\documentclass[11pt]{article}

\usepackage{microtype}
\usepackage{graphicx}
\usepackage{subfigure}
\usepackage{booktabs} % for professional tables
\usepackage{natbib}
\usepackage{epsfig,epsf,fancybox}
\usepackage{amsmath}
\usepackage{mathrsfs}
\usepackage{amssymb}
\usepackage{color}
\usepackage{multirow}
\usepackage{paralist}
\usepackage{verbatim}
\usepackage{algorithm}
\usepackage{algorithmic}
\usepackage{galois}
\usepackage{boxedminipage}
\usepackage{accents}
\usepackage{stmaryrd}
\usepackage[bottom]{footmisc}

\usepackage{natbib}

\usepackage{url}% for url's in bib
\usepackage[colorlinks,linkcolor=magenta,citecolor=blue, pagebackref=true,backref=true]{hyperref}
\renewcommand*{\backrefalt}[4]{%
    \ifcase #1 \footnotesize{(Not cited.)}%
    \or        \footnotesize{(Cited on page~#2.)}%
    \else      \footnotesize{(Cited on pages~#2.)}%
    \fi}

\newtheorem{theorem}{Theorem}[section]
\newtheorem{corollary}[theorem]{Corollary}
\newtheorem{lemma}[theorem]{Lemma}
\newtheorem{proposition}[theorem]{Proposition}

\newtheorem{definition}{Definition}[section]

\newtheorem{remark}[theorem]{Remark}
\newtheorem{assumption}[theorem]{Assumption}

\usepackage{hyperref}

\newcommand{\diag}{\textnormal{diag}}

\newcommand{\var}{\textnormal{var}}

\newcommand{\argmin}{\mathop{\rm argmin}}

\newcommand{\ba}{\begin{array}}
\newcommand{\ea}{\end{array}}

\newcommand{\red}{\color{red}}

\usepackage{bm}
\def\beq{\begin{equation} }\def\eeq{\end{equation} }\def\ep{\varepsilon}\def\1{\mathbf{1}}\def\var{{\rm var}}
\newcommand{\uu}{\bm{u}}
\newcommand{\vv}{{\bm{v}}}
\newcommand{\cC}{\mathcal{C}}
\def\bzeta{\boldsymbol{\zeta}}
\newcommand{\Exs}{\mathbb{E}}
\newcommand{\Ab}{\mathbf{A}}
\newcommand{\Bb}{\mathbf{B}}
\newcommand{\real}{\mathbb{R}}
\newcommand{\cS}{{\mathcal{S}}}
\newcommand{\Ib}{\mathbf{I}}
\newcommand{\cF}{{\mathcal{F}}}
\newcommand{\cH}{{\mathcal{H}}}
\newcommand{\cT}{{\mathcal{T}}}
\newcommand{\cM}{\mathcal{M}}
\def\bSigma{\boldsymbol{\mathbf{\Sigma}}}
\def\lipr{\delta}
\newcommand{\cV}{\mathcal{V}}
\def\rhoi{\rho}
\def\alphai{\mu}
\def\eps{\epsilon}
\newcommand{\PP}{\mathbb{P}}
\newcommand{\bP}{\bm{P}}
\def\ivv{\overline{\vv}}
\newcommand{\Xb}{\bm{X}}
\newcommand{\Yb}{\bm{Y}}
\def\deltai{\delta}
\newcommand{\cN}{\mathcal{N}}
\def\bxi{\boldsymbol{\xi}}
\newcommand{\bR}{\bm{R}}
\newcommand{\bQ}{\bm{Q}}
\def\bchi{\boldsymbol{\chi}}
\newcommand{\bS}{\bm{S}}
\def\tbS{\widetilde{\bS}}
\def\tbP{\widetilde{\bP}}
\def\tDelta{\overline{\Delta}}
\def\radius{r}
\newcommand{\bPhi}{\bm{\mathbf{\Phi}}}
\def\betai{\beta}
\def\gammai{\gamma}
\newcommand{\cA}{\mathcal{A}}
\newcommand{\eb}{\mathbf{e}}
\newcommand{\cL}{\mathcal{L}}
\newcommand{\cQ}{\mathcal{Q}}
\newcommand{\cI}{\mathcal{I}}
\newcommand{\ww}{\bm{w}}
\newcommand{\ud}{d}
\def\kJ{\mathscr{J}}
\newcommand{\cO}{\mathcal{O}}

\usepackage{enumitem}
\def\red#1{}\def\pb{}\usepackage{fullpage}

\begin{document}

%%%%%%% TITLE PAGE %%%%%%%%%%%%%%%%%%%%%%%%%%%%%%%%%%%%%%%%%%%%%%%%%%%

\begin{center}

{\bf{\LARGE{Nonconvex Stochastic Scaled-Gradient Descent and Generalized Eigenvector Problems}}}

\vspace*{.2in}
{\large{
\begin{tabular}{cccc}
Chris Junchi Li$^{\diamond}$
&
Michael I.~Jordan$^{\diamond, \dagger}$ \\
\end{tabular}
}}

\vspace*{.2in}

\begin{tabular}{c}
Department of Electrical Engineering and Computer Sciences$^\diamond$\\
Department of Statistics$^\dagger$ \\ 
University of California, Berkeley
\end{tabular}

\vspace*{.2in}

\today

\vspace*{.2in}

\begin{abstract}
Motivated by the problem of online canonical correlation analysis, we propose the \emph{Stochastic Scaled-Gradient Descent} (SSGD) algorithm for minimizing the expectation of a stochastic function over a generic Riemannian manifold.
SSGD generalizes the idea of projected stochastic gradient descent and allows the use of scaled stochastic gradients instead of stochastic gradients. In the special case of a spherical constraint, which arises in generalized eigenvector problems, we establish a nonasymptotic finite-sample bound of $\sqrt{1/T}$, and show that this rate is minimax optimal, up to a polylogarithmic factor of relevant parameters. 
On the asymptotic side, a novel trajectory-averaging argument allows us to achieve local asymptotic normality with a rate that matches that of Ruppert-Polyak-Juditsky averaging. 
We bring these ideas together in an application to online canonical correlation analysis, deriving, for the first time in the literature, an optimal one-time-scale algorithm with an explicit rate of local asymptotic convergence to normality.
Numerical studies of canonical correlation analysis are also provided for synthetic data.
\end{abstract}
\end{center}


\pb\section{Introduction}\label{sec:intro}
Nonconvex optimization has become the algorithmic engine powering many recent developments in statistics and machine learning. 
Advances in both theoretical understanding and algorithmic implementation have motivated the use of nonconvex optimization formulations with very large datasets, and the striking empirical discovery is that nonconvex models can be successful in this setting, despite the pessimism of classical worst-case analysis.
In this paper, we consider the following general constrained nonconvex optimization problem:
\beq\label{opt_non}
\min_{\vv} F(\vv) , \qquad \textnormal{subject to}~
\vv \in \cC
,
\eeq
where $F(\vv)$ is a smooth and possibly nonconvex objective function and $\cC$ is a feasible set. 
The workhorse algorithm in this setting is stochastic gradient descent (SGD) and its variants \citep{robbins1951stochastic, Qian1999On, duchi2011adaptive, kingma2015adam, zhang2016first}. 
Given an unbiased estimate $\widetilde {\nabla} F(\vv;\bzeta)$ of the gradient $\nabla F(\vv)$, SGD performs the following update at the $t$-th step ($t\ge 1$):
\beq\label{SGD}
\vv_t
	=
\Pi_{\cC}\left[ \vv_{t-1} - \eta \widetilde {\nabla} F(\vv_{t-1};\bzeta_t)
\right],
\eeq
where $\eta > 0$ is a step size and $\Pi_{\cC}$ is a projection operator onto the feasible set $\cC$. 
SGD updates use only a single data point, or a small number of data points, and thus significantly reduce computational and storage complexities compared with offline algorithms, which require storing the full data set and evaluating the full gradient at each iteration.

In many applications, however, we do \textit{not} have access to an unbiased estimate of $\nabla F(\vv)$ when we restrict access to a small number of data points. 
Instead, for each $\vv \in \mathcal{C}$ we have access only to a stochastic vector $\Gamma(\vv;\bzeta)$ which is an unbiased estimate of some \emph{scaled} gradient:
\beq\label{EEG}
\Exs_{\bzeta} \big[ \Gamma(\vv;\bzeta) \big]
	=
D(\vv) \nabla F(\vv)
,
\eeq
where $D(\vv)$ is a deterministic positive scalar that depends on the current state $\vv$. 
Examples of this setup arise most notably in generalized eigenvector (GEV) computation, which finds its applications in principal component analysis, partial least squares regression, Fisher's linear discriminant analysis, canonical correlation analysis (CCA), etc.
Despite this wide range of applications, and their particular relevance to large-scale machine learning problems, there exist few rigorous general frameworks for SGD-based online learning using such models.

Our approach is a conceptually straightforward extension of SGD. 
We propose to continue to use~\eqref{SGD} but with $\widetilde {\nabla} F(\vv_{t-1};\bzeta_t)$ replaced by $\Gamma(\vv_{t-1};\bzeta_t)$. 
We refer this algorithm as the \emph{Stochastic Scaled-Gradient Descent} (SSGD) algorithm.
Specifically, at each step, SSGD performs the update:
\beq\label{SSGD}
\vv_t
=
\Pi_{\cC}\left[
\vv_{t-1} - \eta \Gamma(\vv_{t-1};\bzeta_t)
\right].
\eeq
We provide a theoretical analysis of this algorithm. 
While some of our analysis applies to the algorithm in full generality, our most useful results arise when we specialize to the online GEV problem. 
In this case we aim to minimize the generalized Rayleigh quotient given a unit spherical constraint:
\beq\label{GEV}
\min_{\vv}~
-\frac{\vv^\top \Ab \vv}{\vv^\top\Bb\vv}
	, \qquad \textnormal{subject to}~
\vv \in \real^d,\ \|\vv\| = 1
.
\eeq
The first-order derivative of the generalized Rayleigh quotient with respect to $\vv$ is
\beq\label{gradient_R}
\nabla_\vv \left[ -\frac{\vv^\top \Ab \vv}{\vv^\top \Bb \vv} \right]
	=
-\frac{(\vv^\top \Bb \vv) \Ab \vv - (\vv^\top \Ab \vv) \Bb \vv}{(1/2) (\vv^\top \Bb \vv)^2}
.
\eeq
As pointed out by \citet{arora2012stochastic}, the major stumbling block in applying SGD to this problem lies in obtaining an unbiased stochastic sample of the gradient~\eqref{gradient_R}, due to the fact that the objective function takes a fractional form of two expectations.
In our approach we circumvent this issue by simply replacing the denominator on the right-hand side of~\eqref{gradient_R} by the constant 1 and using the following update:
\beq\label{eq:gev}
\begin{aligned}
\vv_t
&=
\Pi_{\cS^{d-1}} \left[
\vv_{t-1}
+
\eta \left(
	(\vv_{t-1}^\top \widetilde \Bb' \vv_{t-1}) \widetilde \Ab \vv_{t-1} 
    -
    (\vv_{t-1}^\top \widetilde \Ab \vv_{t-1}) \widetilde \Bb' \vv_{t-1} 
\right)
\right]
.
\end{aligned}
\eeq
We refer to the rule \eqref{eq:gev} as an \textit{online GEV iteration}. 
In the special case where $\widetilde{\Bb}$ is taken as $\Ib$, \eqref{eq:gev} essentially reproduces Oja's online PCA algorithm \citep{oja1982simplified} with an incurred $O(\eta^2)$ error term.

To identify the iterative algorithm in \eqref{eq:gev} as a manifestation of SSGD, we rewrite the term in parentheses in the algorithm as follows (we set $\vv = \vv_{t-1}$ for brevity):
\beq\label{eq:gev2}
\begin{aligned}
(\vv^\top \widetilde \Bb' \vv) \widetilde \Ab \vv - (\vv^\top \widetilde \Ab \vv) \widetilde \Bb' \vv
=
\frac{(\vv^\top \Bb \vv)^2}{2}
\cdot
\frac{
(\vv^\top \widetilde \Bb' \vv) \widetilde \Ab \vv - (\vv^\top \widetilde \Ab \vv) \widetilde \Bb' \vv
}{(1/2) (\vv^\top \Bb \vv)^2}
.
\end{aligned}
\eeq
To proceed, we take $\widetilde {\Ab}$ and $\widetilde {\Bb}'$ as mutually independent and unbiased stochastic samples of $\Ab$ and $\Bb$ respectively.
It can be easily seen that the expectation of \eqref{eq:gev2} is a scaled gradient of the generalized Rayleigh quotient, where the scaling is the factor $(\vv^\top \Bb \vv)^2/2$.
This approach, which has been referred to as \textit{double stochastic sampling} in the setting of kernel methods~\citep{dai2014scalable, dai2017learning}, makes it possible to develop an efficient stochastic approximation algorithm. 
Indeed, often $\widetilde{\Ab}, \widetilde{\Bb}$ are of rank one, so the computation of matrix-vector products $\tilde{\Ab} \vv, \tilde{\Bb}' \vv$ only invokes vector-vector inner products and is hence efficient.

Our contributions relative to previous work on nonconvex stochastic optimization as are follows. 
First, we propose a novel algorithm---the stochastic scaled-gradient descent (SSGD) algorithm---which generalizes the classical SGD algorithm and has a wider range of applications.
Second, we provide a local convergence analysis for spherical-constraint objective functions that are locally convex.
Starting with a warm initialization, our local convergence rate matches a known information-theoretic lower bound \citep{mei2018landscape}.
Third, by applying SSGD to the GEV problem, we give a positive answer to the question raised by~\citet{arora2012stochastic} regarding to the existence of an efficient online GEV algorithm.
Specifically, in the case of CCA, our SSGD algorithm uses as few as two samples at each update, does not incur intermediate and expensive computational cost while achieving a polynomial convergence rate guarantee.

\pb\subsection{Related Literature}
The generalized eigenvector problem is at the core of many statistical problems such as principal component analysis~\citep{pearson1901lines,hotelling1933analysis}, canonical correlation analysis \citep{hotelling1936relations}, Fisher's linear discriminant analysis \citep{fisher1936use,welling2005fisher}, partial least squares regression~\citep{stone1990continuum}, sufficient dimension reduction \citep{li1991sliced}, mixture models~\citep{balakrishnan2017statistical}, along with their sparse counterparts.
Iterative algorithms for sparse principal component analysis has been proposed by \cite{ma2013sparse} and \cite{yuan2013truncated} as a special case of the eigenvalue problem:
by adding a soft-thresholding step to each power method step their algorithms achieve linear convergence.
In follow-up work, \cite{tan2018sparse} proposed a truncated Rayleigh flow algorithm to estimate the leading sparse generalized eigenvector that also achieves a linear convergence rate.
Additional work on generalized eigenvector computation includes \cite{ge2016efficient,allen2017doubly,yuan2019decomposition,ma2015finding,chaudhuri2009multi}.

Some recent work has focused on developing efficient online procedures for particular instances of generalized eigenvector problems, among which online principal and canonical eigenvectors estimation has been of particular interest.
Oja's online PCA iteration \citep{oja1982simplified}, which can be reproduced from \eqref{eq:gev} when $\widetilde{\Bb}$ is taken as $\Ib$ as a special case, up to an incurred $O(\eta^2)$ error term, has been shown to provably match the minimax information lower bound \citep{jain2016streaming, li2018near, allen2017first}.
There is also a rich literature on stochastic gradient methods for convex and nonconvex minimization that takes place on Riemannian manifolds \citep{ge2015escaping,zhang2016first};
we refer the readers to \citet{hosseini2020recent} for a recent survey study.
More related to our work, procedures for efficient online canonical eigenvectors estimation have been explored \citep{arora2017stochastic,gao2019stochastic,chen2019constrained}.
Among these works, \citet{gao2019stochastic} developed a streaming canonical correlation analysis (CCA) algorithm which involves solving a large linear system at each iteration, and independently~\citet{arora2017stochastic} proposed a different stochastic CCA algorithm which has temporal and spatial complexities that are quadratic in $d$. 
\citet{chen2019constrained} present a landscape analysis of GEV/CCA and provide a continuous-time insight for a class of primal-dual algorithms when the two matrices in GEV commute; the convergence analysis of \citet{chen2019constrained}, however, does \textit{not} directly translate to discrete-time convergence rate bounds and no explicit analysis has been provided when two matrices do \textit{not} commute.

In a recent paper, \citet{bhatia2018gen} studied the CCA problem and proposed a two-time-scale online iteration that they refer to as ``Gen-Oja.''
The notion of two-time-scale analysis has been used widely in stochastic control and reinforcement learning \citep{BORKAR,KUSHNER-YIN}, and the slow process in Gen-Oja is essentially Oja's iteration \citep{oja1982simplified} for online principal component estimation with Markovian noise \citep{shamir2016convergence,jain2016streaming,li2018near,allen2017first}.
\citet{bhatia2018gen} obtained a convergence rate under a bounded sample assumption that achieves the minimax rate $1/\sqrt{N}$ in terms of the sample size $N$.
In comparison, our proposed SSGD algorithm is a single time-scale algorithm with a single step size and an extra requirement of two (independent) samples per iterate. 
The algorithm is minimax optimal with respect to local convergence and hence theoretically comparable with Gen-Oja.

\pb\subsection{Organization}
The rest of this paper is organized as follows. \S\ref{ssec:assu} states our settings and assumptions throughout the theoretical analysis of our paper. \S\ref{sec_local} presents our local convergence results under the warm initialization condition.
\S\ref{sec_global} presents our two-phase convergence results for arbitrary initialization.
\S\ref{sec_prj} investigates the asymptotic property of our algorithm.
\S\ref{sec:strictsaddle} uses the example of Canonical Correlation Analysis to demonstrate the practical computation and experimental performance of our algorithm.
\S\ref{sec_proof-gev} presents the proofs of our theoretical analysis.
\S\ref{sec_summary} summarizes the entire paper.
Limited by space we relegate to Appendix all secondary lemmas.

\pb\subsection{Notation}
Unless indicated otherwise, $C$ denotes some positive, absolute constant which may change from line to line. 
For two sequences $\{a_n\}$ and $\{b_n\}$ of positive scalars, we denote $a_n \gtrsim b_n$ (resp.~$a_n \lesssim b_n$) if $a_n \ge C b_n$ (resp.~$a_n \le C b_n$) for all $n$, and $a_n \asymp b_n$ if $a_n \gtrsim b_n$ and $a_n \lesssim b_n$ hold simultaneously. 
We also write $a_n = O(b_n), a_n = \Theta(b_n), a_n = \Omega(b_n)$ as $a_n\lesssim b_n, a_n \asymp b_n, a_n \gtrsim b_n$, respectively.
We use $\|\vv\|$ to denote the $\ell_2$-norm of $\vv$. 
Let $\lambda_{\max}(\Ab)$, $\lambda_{\min}(\Ab)$ and $\|\Ab\|$ denote the maximal, minimal eigenvalues and the operator norm of a real symmetric matrix $\Ab$.
We will explain other notation at its first appearance.

\pb\section{Settings and Assumptions}\label{ssec:assu}
In this section, we present the settings and assumptions required by our theoretical analysis of the SSGD algorithm for nonconvex optimization.
To illustrate the core idea we focus on the case of a spherical constraint, $\vv\in\cS^{d-1}$, in which case our proposed SSGD iteration \eqref{SSGD} reduces to the following update:
\beq\label{PSSGD}
\vv_t		=	\Pi_{\cS^{d-1}}\left[ \vv_{t-1} - \eta \Gamma(\vv_{t-1};\bzeta_t) \right]
.
\eeq
Let $\cF_t = \sigma\big( \bzeta_s: s\le t \big)$ be the filtration generated by the stochastic process $\bzeta_t$.
Then, from \eqref{EEG}, we have $\Exs[ \Gamma(\vv_{t-1};\bzeta_t) \mid \cF_{t-1} ] = D(\vv_{t-1}) \nabla F(\vv_{t-1})$.
That is, the conditional expectation is a scaled gradient. 
The ensuing analysis is analogous to that of locally convex SGD given we have appropriate Lipschitz-smoothness of the scalar function $D(\vv)$, but it requires delicate treatment given that SSGD effectively has a varying step size embodied in the scaling factor.

Following the classical theory of constrained optimization \citep{NOCEDAL-WRIGHT} we introduce a definition of \textit{manifold gradient} and \textit{manifold Hessian} in the presence of a unit spherical constraint, $\cC: c(\vv)= \frac12\left(\|\vv\|^2 - 1\right)=0$.%
\footnote{Here for notational simplicity we incorporate a factor of $1/2$.}
For this equality-constrained optimization problem, we utilize the method of Lagrange multipliers and introduce the following Lagrangian function:
$$
L(\vv; \mu) = F(\vv) - \frac{\mu}{2}\left(\|\vv\|^2 - 1 \right)
.
$$
We define the manifold gradient:
\beq\label{gvb}
g(\vv) = \nabla L(\vv; \mu)\big|_{\mu = \mu^*(\vv)}
=
\nabla F(\vv) - \frac{\vv^\top \nabla F(\vv)}{\|\vv\|^2} \vv
,
\eeq
and the manifold Hessian:
\beq\label{Hvb}
\cH(\vv)
 =
\nabla^2 L(\vv; \mu)\big|_{\mu = \mu^*(\vv)}
=
\nabla^2 F(\vv) - \frac{\vv^\top \nabla F(\vv)}{\|\vv\|^2} \Ib
,
\eeq
where $\mu^*(\vv) = \|\vv\|^{-2} \vv^\top \nabla F(\vv)$ is the \textit{optimal Lagrangian multiplier} defined by
$$\begin{aligned}
\frac{\vv^\top \nabla F(\vv)}{\|\vv\|^2}
&=
\argmin_\mu \left\| \nabla L(\vv;\mu) \right\|
=
\argmin_\mu \left\| \nabla F(\vv) - \mu \vv \right\|
.
\end{aligned}$$
For $\vv \in \cS^{d-1}$, we let $\cT(\vv) = \{\uu: \uu^\top\vv = 0\}$ denote the tangent space of $\cS^{d-1}$ at $\vv$.

To prove our main theoretical result, we need the following definitions and assumptions. 
We first define the Lipschitz continuity for a generic mapping:

\begin{definition}[Lipschitz Continuity]\label{defi:Lipschitz}
Let $\cM$ be a finite-dimensional normed vector space. 
The map $M: \real^d \mapsto \cM$ is called $L_M$-Lipschitz, if for any two points $\vv_1, \vv_2 \in \real^d$
$$
\| M(\vv) - M(\vv') \|_{\cM}
\le
L_M \| \vv - \vv' \|,
$$
where $\|\cdot\|_\cM$ is any norm properly defined in space $\cM$.
\end{definition}

In addition, we need the following assumption on the state-dependent scalar $D(\vv)$ and covariance matrix $\bSigma(\vv)$.
For a fixed $\vv$, define the state-dependent covariance $\bSigma(\vv)$ to be
\beq\label{eq:Sigma}
\begin{aligned}
\bSigma(\vv)
&= 
\var\left( \Gamma(\vv;\bzeta) \right)
=
\Exs\left[
	\big( \Gamma(\vv;\bzeta) - D(\vv) \nabla F(\vv) \big)
	\big( \Gamma(\vv;\bzeta) - D(\vv) \nabla F(\vv) \big)^\top
\right]
.
\end{aligned}\eeq
For the purposes of our analysis, we assume that the state-dependent parameter $D(\vv)$ and the Hessian $\nabla^2 F(\vv)$ are Lipschitz continuous within $\{\vv: \|\vv\| \le 1, \|\vv - \vv^*\| \le \lipr\}$, where $\vv^*$ is a local minimizer of the constrained optimization problem \eqref{GEV} and where $\lipr \in (0, 1]$ is a fixed constant.
Within this convex bounded compact space, we can also show that $F(\vv)$ and $\nabla F(\vv)$ are Lipschitz continuous.
We explicitly specify these constants in the following assumption.

\begin{assumption}[Smoothness Assumption]\label{ass:lda}
For any $\vv \in \{\vv: \|\vv\| \le 1, \|\vv - \vv^*\| \le \lipr\}$, we assume that $D(\vv)$ is $L_D$-Lipschitz, $F(\vv)$ is $L_F$-Lipschitz, $\nabla F(\vv)$ is $L_K$-Lipschitz and $\nabla^2 F(\vv)$ is $L_Q$-Lipschitz, where $L_D, L_F, L_K, L_Q$ are fixed positive constants.
\end{assumption}

Now we pose some tail behavior of the stochastic vectors $\Gamma(\vv_{t-1}; \bzeta_t), t\ge 1$ as \textit{vector $\alpha$-sub-Weibull}, as in the following assumption:

\begin{assumption}[Sub-Weibull Tail]\label{ass:se}
For some fixed $\alpha \in (0,2]$ and for all $\vv\in \cC$, we assume that the stochastic vectors $\Gamma(\vv;\bzeta)$ satisfy
$$
\Exs\exp\left(
\frac{ \left\|\Gamma(\vv;\bzeta)\right\|^\alpha }{\cV^\alpha}
\right) \le 2
,
$$
where $\cV$ is called the \emph{sub-Weibull parameter} of stochastic vector $\Gamma(\vv;\bzeta)$.
\end{assumption}
Note here the sub-Weibull parameter is in the vector-norm sense instead of the maximal projection sense.
The class of sub-Weibull distributions contains the  sub-Gaussian ($\alpha=2$) and sub-Exponential ($\alpha=1$) distribution classes as special cases \citep{wainwright2019high,kuchibhotla2018moving}.
Background on vector $\alpha$-sub-Weibull distributions (and the associated notion of Orlicz $\psi_\alpha$-norm) are provided in Appendix~\S\ref{sec:orlicz}.

\pb\section{Local Convergence Analysis}\label{sec_local}
In this section we provide the main local convergence result for our SSGD algorithm.
Our local analysis is inspired from both generic \citep{ge2015escaping} and dynamics-based \citep{li2018near,li2021stochastic} analyses for nonconvex stochastic gradient descent, which we further adapt to our scaled-gradient setup.

For notational simplicity, we denote
\beq\label{rhoi}
D = D(\vv^*)
,
\qquad
\rhoi = D \left( 2L_Q + \frac52 L_F + \frac92 L_K \right)
+
L_D (L_K + 2 L_F)
.
\eeq
For our local convergence analysis, we assume that the initialization $\vv_0$ falls into the neighborhood of a local minimizer $\vv^*$ of the constrained optimization problem; that is,
\beq\label{eq:warm}
\|\vv_0 - \vv^*\|
	\le
\min\left\{ \frac{D \alphai}{2^5 \rhoi}, \lipr \right\}
,
\eeq
where $\alphai$ denotes the minimum positive eigenvalue of the manifold Hessian $\cH(\vv^*)$:
$$
\vv_1^\top \cH(\vv^*) \vv_1 \ge \alphai
	,\qquad
\forall \vv_1\in \cT(\vv^*) ~\textnormal{and}~\|\vv_1\|=1
.
$$
We note that the initialization condition \eqref{eq:warm} has a constant neighborhood radius that does not depend on dimension $d$.
In the ensuing Theorem \ref{theo_local} on local convergence, we take $\eps \in (0, 1)$ and define the following quantities:
\beq\label{Keta}
K_{\eta, \epsilon}
	\equiv
\left\lceil \log_2 \left\{
\frac{\sqrt{D^3 \alphai^3}}{2^5 \rhoi \cV \log^{\frac{\alpha + 2}{2\alpha}} \eps^{-1} \cdot \eta^{1 / 2}} 
\right\} \right\rceil
+ 1
,
\eeq
and for $\eta < 1/(D\alphai)$, define
\beq\label{Tstareta}
T_\eta^* 
	\equiv
\left\lceil
	\frac{2 \log 2}{-\log ( 1 - D\alphai \eta) }
\right\rceil
.
\eeq
We state our local convergence theorem.

\begin{theorem}[Local Convergence]\label{theo_local}
Given Assumptions \ref{ass:lda} and \ref{ass:se} as well as the initialization condition \eqref{eq:warm}, for any positive constants $\eta, \eps$ that satisfy the scaling condition
\beq\label{eq:eta_scaling}
\eta
\le
\min\left\{
\frac{D^3 \alphai^3}{2^{24} G_\alpha^2 \cV^2 \rhoi^2} \log^{-\frac{\alpha + 2}{\alpha}}\eps^{-1}
,~
\frac{1}{D \alphai}
\right\}
,
\eeq
and for any $T \ge K_{\eta, \epsilon} T_\eta^*$, there exists an event $\cH_{\ref{theo_local}}$ with
\beq\label{eq:tailprob}
\PP(\cH_{\ref{theo_local}})
 \ge
1 - \left(
14 + 8 \left(\frac{3}{\alpha}\right)^{\frac{2}{\alpha}} \log^{- \frac{\alpha + 2}{\alpha}} \eps^{-1}
\right) T \eps
,
\eeq
such that on event $\cH_{\ref{theo_local}}$ the iterates generated by the SSGD algorithm satisfy for all $t \in [K_{\eta, \epsilon} T_\eta^*, T]$:
$$
\|\vv_t - \vv^*\|
\le
\frac{2^{\frac{17}{2}} G_\alpha \cV}{\sqrt{D \alphai}} \log^{\frac{\alpha + 2}{2\alpha}}\eps^{-1} \cdot \eta^{1 / 2}
,
$$
where $G_\alpha \equiv \log_2^{1 / \alpha}( 1 + e^{1 / \alpha})\left(1 + \log_2^{1 / \alpha}(1 + e^{1 / \alpha}) \right)$ is a positive factor depending on $\alpha$.
\end{theorem}
To prove Theorem \ref{theo_local}, we define $\Delta_t$ as the projection of $\vv_t - \vv^*$ onto the tangent space $\cT(\vv^*)$, namely
$$
\Delta_t
=
(\Ib-\vv^*{\vv^*}^\top)(\vv_t-\vv^*)
.
$$
We view every $T_\eta^* = \Theta\left( (D \alphai)^{-1} \eta^{-1} \right)$ iterations as one round and interpret $K_{\eta, \epsilon} = \Theta\left( \log \eta^{-1} \right)$ as the number of rounds.  Note that $K_{\eta, \epsilon} T_\eta^*$ can be interpreted as the burn-in time for $\vv_t$ to arrive in a $O(\eta^{1 / 2})$ neighborhood of local minimizer $\vv^*$.
We present a proposition that provides an upper bound on $\|\Delta_t\|$ over $T$ iterations and characterizes the descent in $\|\Delta_t\|$ at the end of each round:

\begin{proposition}\label{prop:localconvexity}
Assume Assumptions \ref{ass:lda}, \ref{ass:se} and initialization condition \eqref{eq:warm} hold.
For any positive constants $\eta, \eps$ satisfying the scaling condition \eqref{eq:eta_scaling} and $T \ge 1$, with probability at least
$$
1 - \left(
14 + 8 \left(\frac{3}{\alpha}\right)^{\frac{2}{\alpha}} \log^{- \frac{\alpha + 2}{\alpha}} \eps^{-1}
\right) T \eps
,
$$
the algorithm iterates satisfy, for all $t \in [0, T]$,
\beq\label{eq:prop-Delta}
\|\Delta_t\|
\le
\|\vv_t - \vv^*\|
\le
\sqrt{2} \|\Delta_t\|
,
\eeq
and
\beq\label{eq:prop-maintain}
\|\Delta_t\|
\le
4 \max\left\{
\frac{\|\Delta_0\|}{2}
,~
\frac{2^6 G_\alpha \cV}{\sqrt{D \alphai}} \log^{\frac{\alpha + 2}{2\alpha}} \eps^{-1} \cdot \eta^{1 / 2}
\right\}
.
\eeq
Moreover, if $T_\eta^* \in [0, T]$, we have:
\beq\label{eq:prop-halve}
\|\Delta_{T_\eta^*}\|
\le
\max\left\{
\frac{\|\Delta_0\|}{2}
,~
\frac{2^6 G_\alpha \cV}{\sqrt{D \alphai}} \log^{\frac{\alpha + 2}{2\alpha}} \eps^{-1} \cdot \eta^{1 / 2}
\right\}
.
\eeq
\end{proposition}
The proof of Proposition \ref{prop:localconvexity} is provided in \S\ref{sec:localconvexity}.

By choosing an asymptotic regime such that $T \epsilon \log (1/\ep) \to 0$, Proposition \ref{prop:localconvexity} states that \eqref{eq:prop-Delta}, \eqref{eq:prop-maintain} and \eqref{eq:prop-halve} hold with probability tending to one.
On that high-probability event, \eqref{eq:prop-Delta} indicates that $\|\vv_t - \vv^*\|$ and its projection in the tangent space $\|\Delta_t\|$ are bounded by each other up to constant factors,
\eqref{eq:prop-maintain} guarantees that $\|\Delta_t\|$ does not exceed $\max\left\{ 2\|\Delta_0\|, \Theta(\eta^{1 / 2})\right\}$---that is, $\vv_t$ stays in a neighborhood of local minimizer $\vv^*$---and \eqref{eq:prop-halve} states that, for $\|\Delta_0\| = \Omega(\eta^{1 / 2})$, $\|\Delta_t\|$ decreases by half after $T_\eta^*$ iterations: $\|\Delta_{T_\eta^*}\| \le \max\left\{ \|\Delta_0\| / 2, \Theta(\eta^{1 / 2})\right\}$ .

Proposition \ref{prop:localconvexity} studies $\Delta_t$ in a single round, i.e., for $T_\eta^*$ iterations.
We are ready to provide the proof of Theorem \ref{theo_local} by applying Proposition \ref{prop:localconvexity} repeatedly for $K_{\eta, \epsilon}$ rounds, detailed as follows:

\noindent\textbf{Proof of Theorem \ref{theo_local}}
Since the algorithm iteration \eqref{SSGD} can be viewed as a discrete-time (strong) Markov process,
We recall the definition of $K_{\eta, \epsilon}$ in \eqref{Keta} and repeatedly apply Proposition \ref{prop:localconvexity} to the sequence of $\{\Delta_t\}$ for $K_{\eta, \epsilon}$ rounds, initializing each round with the output $\Delta_{T_\eta^*}$ from the previous round.
We adopt an adaptive argument of shrinkage in multiple rounds.

More specifically, for any $t \in [K_{\eta, \epsilon} T_\eta^*, T]$, we first apply \eqref{eq:prop-halve} in Proposition \ref{prop:localconvexity} for $K_{\eta, \epsilon}$ rounds, then apply \eqref{eq:prop-maintain} for $t - K_{\eta, \epsilon} T_\eta^*$ iterations, and use \eqref{eq:prop-Delta} to conclude that 
\begin{align*}
\|\vv_t - \vv^*\|
    &\le
\sqrt{2} \|\Delta_t\|
    \le
\sqrt{2} \cdot 4 \max\left\{
\frac{\|\Delta_{K_{\eta, \epsilon} T_\eta^*}\|}{2}
,
\frac{2^6 G_\alpha \cV}{\sqrt{D \alphai}} \log^{\frac{\alpha + 2}{2\alpha}} \eps^{-1} \cdot \eta^{1 / 2}
\right\}
    \\&\le
4\sqrt{2}
\cdot
\max\left\{\frac{\|\Delta_0\|}{2^{K_{\eta, \epsilon}}}
,
\frac{2^6 G_\alpha \cV}{\sqrt{D \alphai}} \log^{\frac{\alpha + 2}{2\alpha}} \eps^{-1} \cdot \eta^{1 / 2}
\right\}
    \le
\frac{2^{\frac{17}{2}} G_\alpha \cV}{\sqrt{D \alphai}} \log^{\frac{\alpha + 2}{2\alpha}} \eps^{-1} \cdot \eta^{1 / 2}
, 
\end{align*}
where the last inequality is due to initialization condition \eqref{eq:warm}.
Here $G_\alpha$ is a fixed positive factor depending on $\alpha$, as defined in Theorem~\ref{theo_local}.
By taking a union bound over $K_{\eta, \epsilon}$ rounds and $T - K_{\eta, \epsilon} T_\eta^*$ iterations, we obtain
$$
\PP(\cH_{\ref{theo_local}})
\ge
1 - \left(
14 + 8 \left(\frac{3}{\alpha}\right)^{\frac{2}{\alpha}} \log^{- \frac{\alpha + 2}{\alpha}} \eps^{-1}
\right) T \eps
,
$$
completing the proof of Theorem \ref{theo_local}.

Theorem \ref{theo_local} establishes the local convergence of $\vv_t$ in a neighborhood of $\vv^*$ for a fixed step size $\eta$ and a number of iterations $T \ge K_{\eta, \epsilon} T_\eta^*$.
The following corollary provides a finite-sample bound:

\begin{corollary}[Finite-Sample]\label{coro:main}
Assume Assumptions \ref{ass:lda} and \ref{ass:se} and the initialization condition \eqref{eq:warm}.
For fixed positive constants $\eps$ and sample size $T$, set the step size as
$$
\begin{aligned}
&\eta(T)
=
\Theta\left(
\frac{\log T}{D \alphai T}
\right)
\end{aligned}
$$
satisfying scaling condition
$$
\begin{aligned}
&\eta(T)
\le
\min\left\{
\frac{D^3 \alphai^3}{2^{24} G_\alpha \cV^2 \rhoi^2} \log^{-\frac{\alpha + 2}{\alpha}} \eps^{-1}
,~
\frac{1}{D \alphai}
\right\}
,
\end{aligned}
$$
there exists an event $\cH_{\ref{coro:main}}$ with
$$
\PP(\cH_{\ref{coro:main}})
 \ge
1 - \left(
14 + 8 \left(\frac{3}{\alpha}\right)^{\frac{2}{\alpha}} \log^{- \frac{\alpha + 2}{\alpha}} \eps^{-1}
\right) T \eps
,
$$
such that on the event $\cH_{\ref{coro:main}}$ the iterates generated by the SSGD algorithm satisfy
$$
\|\vv_T - \vv^*\|
\lesssim
\frac{G_\alpha \cV}{D \alphai} \log^{\frac{\alpha + 2}{2\alpha}} \eps^{-1} \sqrt{\frac{\log T}{T}}
.
$$
\end{corollary}

We notice that our Theorem \ref{theo_local} and Corollary \ref{coro:main} provide a \emph{dimension-free} local convergence rate when $\cV$ is $O(1)$.
As we will see later in the example of CCA, the ($\alpha=1/2$) sub-Weibull parameter $\cV$ in that case scales with $\sqrt{d}$ and thus the local rate is the minimax-optimal rate $O(\sqrt{d/T})$ up to a polylogarithmic factor.

\pb\section{Global Convergence Analysis}\label{sec_global}
In many situations, solving the warm initialization problem itself can be a difficult problem.
We borrow the techniques from \citet{ge2015escaping} and establish a global convergence result for \textit{escaping saddle points} via SSGD.
In this section we consider a variant of SSGD with a unit spherical constraint and equipped with an artificial noise injection step: 
let $\bm{n}_t$ be an independent spherical noise at each step that is independent of $\cF_{t-1}$ and $\bzeta_t$, and let
\beq\label{SSGDglobal}
\vv_t = \Pi_{\cS^{d-1}}\left[ \vv_{t-1} - \eta \Gamma(\vv_{t-1};\bzeta_t) + \eta \bm{n}_t \right]
.
\eeq
Motivated by recent work on \text{escaping saddle points} \citep{ge2015escaping,lee2016gradient,jin2019nonconvex}, one can show that SSGD algorithm equipped with the aforementioned artificial noise injection escapes from all saddle points, and hence the initialization condition \eqref{eq:warm} can be dropped.

First, we generalize Assumption \ref{ass:lda} for local convergence to the following for global convergence:
\begin{assumption}[Global Smoothness and Boundedness]\label{ass_global}
For any $\vv \in \{\vv: \|\vv\| \le 1\}$, we assume that $D(\vv)$ is $L_D$-Lipschitz, $F(\vv)$ is $L_F$-Lipschitz, $\nabla F(\vv)$ is $L_K$-Lipschitz and $\nabla^2 F(\vv)$ is $L_Q$-Lipschitz.
Also, assume there exists $D_-, D_+>0$ such that $D_-\le D(\vv)\le D_+$ for all $\vv$.
\end{assumption}

\begin{definition}[Strict-Saddle Function]\label{defi_strictsaddle}
A twice differentiable function $F(\vv)$ with constraint $c(\vv) = 0$ is called an $(\alphai,\betai,\gammai,\deltai)$-strict-saddle function, if an arbitrary point $\vv$ with $c(\vv) = 0$ satisfies at least one of the following:

\begin{enumerate}[label=(\roman*)]
\item
$\|g(\vv)\| \ge \betai$;
\item
There is a local minimizer $\vv^*$ such that $\|\vv - \vv^*\| \le \deltai$.
Additionally, for all $\vv' \in B_{2\deltai}(\vv^*)$, we have 
$$
\vv_1^\top \cH(\vv') \vv_1 \ge \alphai
, 
\quad\forall \vv_1\in \cT(\vv') ~\textnormal{and}~\|\vv_1\|=1
.
$$

\item
There exists a unit vector $\vv_0\in \cT(\vv)$ such that $\vv_0^\top \cH(\vv) \vv_0 \le -\gammai$.
\end{enumerate}
\end{definition}

In what follows, we show that our algorithms can escape from all saddle points and thus the local initialization is no longer required.
We are ready to present the saddle-point escaping result:

\begin{theorem}[Escaping from Saddle Points]\label{theo:saddle}
Let Assumptions \ref{ass:se} and \ref{ass_global} hold.
Let $F(\vv)$ be a $(\alphai, \betai, \gammai, \deltai)$-strict-saddle function with finite sup-norm $\|F\|_\infty$.
Let
\beq\label{T1}
T_1 = 
4\|F\|_\infty
 \cdot
\left[
 \min\left(
 0.5 d L_G
 ,
 \gammai 
\log^{-1} \left( \frac{6d\cV}{\sigma} \right)
 \right) \cdot
 \sigma^2 D_-^2 \eta^2
\right]^{-1}
.
\eeq
Then for any $\kappa>0$ and any step size $\eta > 0$ satisfying
\beq\label{etamax}
\sqrt{2d\cV^2 L_G D_+\eta} \le \betai
,
\eeq
within $T_1 \cdot \lceil \log_2 (\kappa^{-1}) \rceil$ iterates, \eqref{SSGDglobal} outputs $\vv_t$ that satisfies (ii) in Definition \ref{defi_strictsaddle} with probability no less than $1-\kappa$.
\end{theorem}
The proof of Theorem \ref{theo:saddle} is collected in \S\ref{sec:escapingSSGD}.
Motivated by this saddle-point escaping result, one can run SSGD first with a \textit{burn-in} phase and once it enters the warm initialization region, one can re-run SSGD with step sizes chosen so that the local convergence theorem applies immediately.
Using the strong Markov property and combining Theorems \ref{theo_local} and \ref{theo:saddle} we immediately obtain the following main theorem.
Recall that $T_1$ is defined as in \eqref{T1}.

\begin{theorem}[Two-Phase Global Convergence]\label{coro:saddle}
Let Assumptions \ref{ass:se} and \ref{ass_global} hold.
Let $\eta$ satisfy
\beq\label{eq:eta_scaling2}
\eta
	\le
\min\left\{
\frac{D^3 \alphai^3}{2^{24} G_\alpha^2 \cV^2 \rhoi^2} \log^{-\frac{\alpha + 2}{\alpha}}\eps^{-1}
,~
\frac{1}{D \alphai}
,~
\frac{\betai^2}{2d\cV^2 L_G D_+}
\right\}
,
\eeq
and for any $T \ge K_{\eta, \epsilon} T^*_\eta + T_1 \cdot \lceil \log_2(\kappa^{-1}) \rceil$, there exists an event $\cA_T$ with
$$
\PP(\cA_T)
	\ge
1 - \kappa - \left(
14 + 8 \left(\frac{3}{\alpha}\right)^{\frac{2}{\alpha}} \log^{- \frac{\alpha + 2}{\alpha}} \eps^{-1}
\right) T \eps
,
$$
such that on event $\cA_T$ the iterates generated by the SSGD algorithm satisfy for all $t\in \left[ K_{\eta, \epsilon} T^*_\eta + T_1 \cdot \lceil \log_2(\kappa^{-1}) \rceil, T\right]$
$$
\| \vv_t - \vv^* \|
	\le
\frac{2^{\frac{17}{2}} G_\alpha \cV}{\sqrt{D \alphai}} \log^{\frac{\alpha + 2}{2\alpha}}\eps^{-1} \cdot \eta^{1 / 2}
,
$$
where $G_\alpha \equiv \log_2^{1 / \alpha}( 1 + e^{1 / \alpha})\left(1 + \log_2^{1 / \alpha}(1 + e^{1 / \alpha}) \right)$ is a positive factor depending on $\alpha$.
\end{theorem}

Note the function class of strict-saddle functions is strictly more general than the local convergence Theorem \ref{theo_local}.
We find the final complexity by interpreting Theorem \ref{coro:saddle}.
In the asymptotic relations below we write out the dependency on $d,\eta$, and let $\cL$ be a generic quantity that only involves a polylogarithmic factor of $d$, $\eta$ and $T$, which is allowed to vary at each appearance.
From \eqref{Keta}, \eqref{Tstareta} and \eqref{T1} we have 
$$
K_{\eta, \epsilon} T^*_\eta		\asymp \cL \cdot \eta^{-1}
	,\qquad
T_1 \cdot \lceil \log_2(\kappa^{-1}) \rceil		\asymp \cL \cdot d^{-1} \eta^{-2},
$$
and if $\cV$ is set as the model scaling $\sqrt{d}$, the iteration achieves a high-probability bound of $\cL\cdot \sqrt{d\eta}$ after $K_{\eta, \epsilon} T^*_\eta + T_1 \cdot \lceil \log_2(\kappa^{-1}) \rceil$ steps.
We conclude that under the scaling condition $\cL \cdot d / T \to 0$, if the total number of samples $T$ is given, we can optimize the choice of step size $\eta = \eta(d,T)$ to conclude the following convergence rate results:

\begin{itemize}
\item
\textbf{Local convergence: }
Given a \textit{warm} initialization, and choosing $\eta(T) \asymp \cL \cdot (1/T)$, SSGD \eqref{SSGD} has the following \textit{local} convergence rate
$$
\| \vv_t - \vv^* \| \lesssim \cL\cdot \sqrt{\frac{d}{T}}
.
$$

\item
\textbf{Global convergence: }
Given \textit{any} initialization, and choosing $\eta(T) \asymp \cL \cdot (1 / \sqrt{dT})$, SSGD with noise injection \eqref{SSGDglobal} has the following \textit{global} convergence rate
$$
\| \vv_t - \vv^* \| \lesssim \cL\cdot \sqrt[4]{\frac{d}{T}}
.
$$
\end{itemize}
We defer the arguments for the proof to \S\ref{sec:escapingSSGD}, and turn to the application to GEV problem.

\pb\subsection{Problem-dependent Parameters for GEV}
We need to verify that the objective function for the GEV problems is indeed in the class of strict-saddle functions.
For the generalized eigenvector problem, the objective function of interest is
\beq\label{optGEV}
F(\vv) = -\frac{\vv^\top \Ab \vv}{\vv^\top \Bb \vv},
\qquad
\textnormal{such that}~~
c(\vv) = \frac12\left(\|\vv\|^2 - 1\right) = 0
,
\eeq
where $\Ab$ and $\Bb$ are two real symmetric matrices with $\Bb$ being strictly positive-definite.
We make one additional mild assumption on the eigenstructure of matrices $\Ab$ and $\Bb$.

\begin{assumption}\label{assu:GEVassu}
The matrix $\Bb^{-1/2} \Ab \Bb^{-1/2}$ is diagonalizable with eigenvalues $\lambda_1 > \lambda_2 > \dots > \lambda_d$.
Moreover, $\lambda_{\min}(\Bb) > 0$.
\end{assumption}
As our argument proceeds, one can safely assume $\Bb^{-1/2} \Ab \Bb^{-1/2}$ being diagonal without loss of generality, and we will assume such unless otherwise specified.
Under Assumption \ref{assu:GEVassu} we denote the minimal gap of $\lambda_i$'s as
\beq\label{lambdagap}
\lambda_{\text{gap}}
 =
\min_{1\le i\le d-1} (\lambda_i - \lambda_{i+1})
> 0
.
\eeq
In the following proposition, we verify that the objective function $F(\vv)$ in \eqref{optGEV} satisfies Assumption \ref{ass:lda}; that is, $D(\vv), F(\vv), \nabla F(\vv), \nabla^2 F(\vv)$ are Lipschitz continuous within $\{\vv: \|\vv\| \le 1, \|\vv - \vv^*\| \le \lipr\}$:
\begin{proposition}\label{prop_gev_lipschitz}
Assumption \ref{ass:lda} holds for $F(\vv)$ in GEV problem \eqref{optGEV} with constants
$$
\begin{aligned}
L_D
&=
2\|\Bb\|^2
,
\quad
L_F
=
\frac{4 \|\Ab\| \|\Bb\|}{(1 - \lipr)^2 \lambda_{\min}^2(\Bb)}
,
\quad
L_K
=
\frac{28 \|\Ab\| \|\Bb\|^2}{(1 - \lipr)^3 \lambda_{\min}^3(\Bb)}
,
\quad
L_Q
=
\frac{232 \|\Ab\| \|\Bb\|^3}{(1 - \lipr)^4 \lambda_{\min}^4(\Bb)}
.
\end{aligned}
$$
\end{proposition}
The proof of Proposition \ref{prop_gev_lipschitz} is deferred to \S\ref{sec_proof,prop_gev_lipschitz}.
With the Lipschitz parameters given above, we consider the initialization condition \eqref{eq:warm}.
The neighborhood radius on the right-hand side of \eqref{eq:warm} can be viewed as a function of $\lipr$ that is maximized at some $\lipr^* \in (0, 1)$, when all other constants are fixed.
The region covered in the local convergence analysis is maximized with such a choice of $\lipr^*$.

Now we prove that under the mild Assumption \ref{assu:GEVassu}, the objective function for generalized eigenvector problem is strict-saddle as in Definition \ref{defi_strictsaddle} if the parameters are chosen properly:

\begin{proposition}\label{prop_GEVstrictsaddle}
Under Assumption \ref{assu:GEVassu}, the only local minimizers of \eqref{optGEV} are $\pm \eb_1$, and the function satisfies the $(\alphai,\betai,\gammai,\deltai)$-strict saddle condition for
\beq\label{GEVparam}
\begin{split}
&
\alphai = (\lambda_1 - \lambda_2) \frac{ \lambda_{\min} (\Bb) }{ \|\Bb\| }
,\qquad
\betai = (\lambda_1 - \lambda_2) \frac{ \lambda_{\min} (\Bb) }{ \|\Bb\| }
,\\&
\gammai
=
\lambda_{\text{gap}}^3
\frac{\lambda_{\min}^8(\Bb)}{(8)84^2\|\Ab\|^2\|\Bb\|^6}
,\qquad
\deltai
 =
(\lambda_1 - \lambda_2)
\frac{ \lambda_{\min}^4(\Bb)}{168\|\Ab\|\|\Bb\|^3}.
\end{split}
\eeq
\end{proposition}
To verify the strict-saddle parameters and conclude Proposition \ref{prop_GEVstrictsaddle}, we first conclude the parameters for the objective function of the eigenvector problem:

\def\xb{\mathbf{x}}
\def\yb{\mathbf{y}}
\def\zb{\mathbf{z}}
\def\bv{\mathbf{v}}
\begin{lemma}\label{lemm_littlegrad}
Under Assumption \ref{assu:GEVassu}, and with the choices of parameters as in \eqref{GEVparam}, we have the following:

\begin{enumerate}[label=(\roman*)]
\item
Suppose $\|g(\xb)\|\le \gammai$ and $|\eb_1^\top \Bb^{1/2} \xb| \le (1/2) \|\Bb^{1/2} \xb\|$.
Let the vector
$$
\vv \equiv \frac{P_{\cT(\xb)} \Bb^{-1/2} \eb_1 }{ \|P_{\cT(\xb)} \Bb^{-1/2} \eb_1\|}
,
$$
then $\vv\in \cT(\xb)$, $\|\vv\| = 1$, and we have
\beq\label{saddle}
\vv^\top \cH(\xb) \vv \le -\betai
.
\eeq

\item
Suppose $\|g(\xb)\|\le \gammai$ and $|\eb_1^\top \Bb^{1/2} \xb| > (1/2) \|\Bb^{1/2} \xb\|$.
Then there is a local minimizer $\xb^*$ such that $\|\xb - \xb^*\| \le \deltai$, and for all $\xb' \in \Bb_{2\deltai}(\xb^*)$ we have for all $\hat \vv\in \cT(\xb')$ and $\|\hat \vv\| = 1$
\beq\label{strcvx}
\hat \vv^\top \cH(\xb') \hat \vv
\ge \alphai
.
\eeq
\end{enumerate}
\end{lemma}
It is straightforward from Definition \ref{defi_strictsaddle} of strict-saddle property that Lemma \ref{lemm_littlegrad} leads to Proposition \ref{prop_GEVstrictsaddle} immediately.
We postpone the details to \S\ref{sec_proof,lemm_littlegrad}.
Intuitively, the parameters are only dependent on the differences of the consecutive (generalized) eigenvalues $\lambda_1 - \lambda_2, \dots, \lambda_{d-1} - \lambda_d$, since we can always shift each $\lambda_i$ by an arbitrary constant and keep the constrained optimization problem \eqref{optGEV} unchanged.
We also remark that restricted to our analysis, the parameters in \eqref{GEVparam} might not be the sharpest possible choices.
However, we do provide, to the best of our knowledge, a first identification of strict-saddle parameters for the GEV problem, and hence Theorems \ref{theo:saddle} and \ref{coro:saddle} apply (given the proper tail conditions of the stochastic noise).

\pb\section{Asymptotic Normality via Trajectory Averaging}\label{sec_prj}
In this section, we return to the warm initialization as in \S\ref{sec_local}.
\citet{ruppert1988efficient} and \citet{polyak1992acceleration} introduced the idea of trajectory averaging for stochastic gradient descent in order to provide fine-grained convergence rates along with an asymptotic normality result.
Our goal is to generalize the Polyak-Juditsky analysis of SGD with trajectory averaging to SSGD for nonconvex objective that is initialized in a local convex region.
We denote $\cH_* \equiv \cH(\vv^*), \bSigma_* \equiv \bSigma(\vv^*)$ and $D \equiv D(\vv^*)$.
Define 
$$
\cM_*
=
(\Ib - \vv^* {\vv^*}^\top) \cH_* (\Ib - \vv^* {\vv^*}^\top)
.
$$
From the initialization condition \eqref{eq:warm}, we have $\uu^\top \cM_* \uu \ge \alphai \|\uu\|^2$ for all $\uu \in \cT(\vv^*)$.
We consider the eigendecomposition $\cM_* = \bP \diag(\lambda_1, \ldots, \lambda_{d - 1}, 0) \bP^\top$ for an orthogonal matrix $\bP \in \real^{d \times d}$ and eigenvalues $\lambda_1 \ge \ldots \ge \lambda_{d - 1} > 0$ with minimum positive eigenvalue $\lambda_{d - 1} \ge \alphai$.
We take the inverse of all positive eigenvalues and define the following matrix
\beq\label{eq:M*_inverse}
\cM_*^-
\equiv
\bP \diag(\lambda_1^{-1}, \ldots, \lambda_{d - 1}^{-1}, 0) \bP^\top
.
\eeq
Here, $\cM_*^-$ can be interpreted as the inverse of $\cM_*$ in the $(d-1)$-dimensional tangent space $\cT(\vv^*)$, and we can easily find $\cM_*^- \vv^* = \bm{0}$.
As shown in Theorem \ref{theo_local}, we need $K_{\eta, \epsilon} T_\eta^*$ iterations for $\vv_t$ to fall in a $\Theta(\eta^{1 / 2})$ neighborhood of the local minimizer $\vv^*$.
For $T \ge K_{\eta, \epsilon} T_\eta^*$, we define the trajectory average over time $K_{\eta, \epsilon} T_\eta^* + 1, \ldots, T$ as follows:
\beq\label{eq:ave}
\ivv_T^{(\eta)}
\equiv
\frac{1}{T - K_{\eta, \epsilon} T_\eta^*} \sum_{t = K_{\eta, \epsilon} T_\eta + 1}^T \vv_t
,
\eeq
where we add the superscript $(\eta)$ to emphasize the dependency on $\eta$.
Notice that $\{\ivv_T^{(\eta)}\}_{T, \eta}$ is a triangular array over a continuum $\eta$.
To obtain asymptotic normality of the trajectory average $\ivv_T^{(\eta)}$, we additionally make the following local Lipschitz-continuity assumption on stochastic scaled-gradient $\Gamma(\vv; \bzeta)$ in the neighborhood of $\vv^*$:
\begin{assumption}\label{ass:ssg-lip}
There exists a positive constant $L_S$ such that for all $\vv, \vv' \in \{\vv : \|\vv\| \le 1, \|\vv - \vv^*\| \le \lipr\}$, we have
\beq\label{eq:ssg_lip}
\Exs\left\| \Gamma(\vv; \bzeta) - \Gamma(\vv'; \bzeta) \right\|^2
  \le
L_S^2 \|\vv - \vv'\|^2
.
\eeq
\end{assumption}

The following theorem states that the trajectory average $\ivv_T^{(\eta)}$ converges in distribution to a $(d-1)$-dimensional normal distribution in the tangent space $\cT(\vv^*)$:

\begin{theorem}[Asymptotic Normality]\label{theo:asympnorm}
Given Assumptions \ref{ass:lda}, \ref{ass:se}, \ref{ass:ssg-lip} and initialization condition \eqref{eq:warm}, if we choose the step size $\eta$ such that $\eta \rightarrow 0$ as the total sample size $T \rightarrow \infty$, where
\beq\label{eq:condition}
T \eta^2 \log^{\frac{2\alpha + 4}{\alpha}} T \rightarrow 0
	,\qquad
T \eta \log^{- \frac{\alpha + 2}{\alpha}} T \rightarrow \infty
,
\eeq
we obtain Gaussian convergence in distribution:
\beq\label{eq:an}
\sqrt{T} \left( \ivv_T^{(\eta)} - \vv^* \right)
	\stackrel{d}{\rightarrow} 
\mathcal{N}\left(\mathbf{0}, D^{-2} \cdot \cM_*^- \bSigma_* \cM_*^- \right)
.
\eeq
\end{theorem}
We relegate the proof details of Theorem \ref{theo:asympnorm} to \S\ref{sec_proof,theo:asympnorm}.%
\footnote{The limiting distribution concentrates on a submanifold of the Euclidean space $\real^d$. The convergence in distribution is hence rigorously characterized by the pointwise convergence of the characteristic functions.}
The analysis has the same rationale as the classical asymptotic normality result that is obtained when minimizing a strongly convex objective function in an Euclidean space using stochastic gradient descent \citep{ruppert1988efficient,polyak1992acceleration}.
Indeed, in the case of a diminishing step size, $\eta(t) \propto t^{-\alpha}$, $\alpha\in (1/2,1)$, SGD with trajectory averaging converges in distribution to a normal distribution.
In contrast, due to our choice of a constant step size that is asymptotically small with $\eta \propto T^{-\alpha}$ up to a polylogarithmic factor, we base our analysis on the idea that trajectory averaging begins only after ``the burn-in phase''; that is, after $K_{\eta,\ep} T_\eta^*$ iterates.

\pb\section{Case Studies of Canonical Correlation Analysis}\label{sec:strictsaddle}
The GEV problem arises in many statistical machine learning tasks.
We focus on the example of (rank-one) Canonical Correlation Analysis (CCA) as a core application;
we refer to \citet{tan2018sparse} for other (sparse, high-dimensional) applications including linear discriminant analysis and sliced inverse regression.
Recall that CCA aims at maximizing the correlation between two transformed vectors.
Given $\Xb$ and $\Yb$ as two column vectors, let $\bSigma_{\Xb\Yb}$ be the cross-covariance matrix between $\Xb$ and $\Yb$, and let $\bSigma_{\Xb\Xb}$ and $\bSigma_{\Yb\Yb}$ be the covariance matrices of $\Xb$ and $\Yb$, respectively. CCA is a special case of the GEV problem~\eqref{GEV} with
$$
\Ab= \begin{pmatrix}
\mathbf{0} & \bSigma_{\Xb\Yb} 
\\
\bSigma_{\Yb\Xb} & \mathbf{0}
\end{pmatrix}
,
\qquad
\Bb= \begin{pmatrix}
\bSigma_{\Xb\Xb} &\mathbf{0}
\\
\mathbf{0} & \bSigma_{\Yb\Yb}
\end{pmatrix}
.
$$
To obtain $\widetilde{\Ab}, \widetilde{\Bb}'$ as mutually independent and unbiased stochastic samples of $\Ab$ and $\Bb$, we draw two independent pairs of samples $(\Xb, \Yb), (\Xb', \Yb')$ at each iteration and compute
$$\begin{aligned}
\widetilde{\Ab}= \begin{pmatrix} 
\mathbf{0}		& \Xb \Yb^\top
 \\
\Yb \Xb^\top	& \mathbf{0} 
\end{pmatrix}
,
\qquad
\widetilde{\Bb}'= \begin{pmatrix}
\Xb' \Xb'^\top	&\mathbf{0}
 \\
\mathbf{0}		& \Yb' \Yb'^\top
\end{pmatrix},
\end{aligned}$$
where all samples of $\Xb, \Yb$ are centered such that they have expectation zero.

\begin{algorithm}[!tb]
\caption{Online Canonical Correlation Analysis via Stochastic Scaled-Gradient Descent (Noisy)}
\begin{algorithmic}
\STATE Given total sample size $T$ and proper step size $\eta$ and initialization $\vv_0$
\FOR{$t = 1, \ldots, T / 2$}
\STATE Draw mutually independent sample pairs $(\Xb, \Yb)$ and $(\Xb', \Yb')$ from the sampling oracle
\STATE Compute unbiased estimates
$$
\widetilde{\Ab}= \begin{pmatrix} 
\mathbf{0}		& \Xb \Yb^\top
\\
\Yb \Xb^\top	& \mathbf{0}  
\end{pmatrix}
\qquad
\widetilde{\Bb}'= \begin{pmatrix}  
\Xb' \Xb'^\top	& \mathbf{0}
\\
\mathbf{0}		& \Yb' \Yb'^\top  
\end{pmatrix}
$$
Sample a uniformly spherical noise $\bm{n}_t$ of covariance $\sigma^2 \Ib_d$ and update $\vv_t$ using the following rule
$$
\vv_t
\leftarrow
\Pi_{\cS^{d - 1}} \left[
\vv_{t - 1}
+
\eta \left(
(\vv_{t - 1}^\top \widetilde{\Bb}' \vv_{t - 1}) \widetilde{\Ab} \vv_{t - 1} - (\vv_{t - 1}^\top \widetilde{\Ab} \vv_{t - 1}) \widetilde{\Bb}' \vv_{t - 1}
\right)
+
\eta \bm{n}_t
\right]
$$
\ENDFOR
\STATE 
Return $\vv_T$
\end{algorithmic}
\label{algo:cca}
\end{algorithm}

In order to apply the convergence results for the SSGD algorithm to the CCA problem, it remains to verify Assumption \ref{ass:se}.
We assume that the samples $\Xb \in \real^{d_x}, \Yb \in \real^{d_y}$ follow sub-Gaussian distributions \citep{gao2019stochastic, li2018near} with parameters $\cV_x, \cV_y$; that is,
$
\Exs\exp\left( \|\Xb\|^2 / \cV_x^2 \right)  \le  2
$ and $
\Exs\exp\left( \|\Yb\|^2 / \cV_y^2 \right)  \le  2
.
$
With these standard assumptions for the samples $\Xb, \Yb$, the following lemma shows that the scaled-gradient noise in the CCA problem satisfies Assumption \ref{ass:se} with appropriate $\cV$ and $\alpha$.
The proof is provided in \S\ref{sec_proof,prop_cca_subweibull}.

\begin{proposition}\label{prop_cca_subweibull}
Assumption \ref{ass:se} holds for CCA with parameters $\cV = 400 (\cV_x^2 + \cV_y^2) \cV_x \cV_y$ and $\alpha = 1/2$.
\end{proposition}
Lemmas \ref{prop_gev_lipschitz} and \ref{prop_cca_subweibull} certify that Assumptions \ref{ass:lda} and \ref{ass:se} hold in CCA settings and hence local convergence Corollary \ref{coro:main} applies, which establishes a $\sqrt{d/T}$-rate up to a polylogarithmic since the vector sub-Weibull parameter $\cV$ in our Assumption \ref{ass:se} implicitly contains a factor $\sqrt{d}$.

Now we demonstrate that our bounds in Corollary \ref{coro:main} match the lower bound.
\citet{gao2019stochastic} derived a lower bound for Gaussian variables, $1 - \mathrm{align}(\vv, \vv^*) \gtrsim d / T$, in terms of a new measure of error:
$$\begin{aligned}
\mathrm{align}(\vv, \vv^*)
    &\equiv
\frac12 \left(
\frac{\vv_x^\top \bSigma_{\Xb \Xb} \vv_x^*}{\sqrt{{\vv_x}^\top \bSigma_{\Xb \Xb} \vv_x} \sqrt{{\vv_x^*}^\top \bSigma_{\Xb \Xb} \vv_x^*}}
+
\frac{\vv_y^\top \bSigma_{\Yb \Yb} \vv_y^*}{\sqrt{{\vv_y}^\top \bSigma_{\Yb \Yb} \vv_y} \sqrt{{\vv_y^*}^\top \bSigma_{\Yb \Yb} \vv_y^*}}
\right)
,
\end{aligned}$$
where $\vv^\top = (\vv_x^\top, \vv_y^\top)$ and ${\vv^*}^\top = ({\vv_x^*}^\top, {\vv_y^*}^\top)$ are partitioned in dimensions $d_x, d_y$.
It is easy to verify that $1 - \mathrm{align}(\vv, \vv^*) \asymp 1 - \vv^\top \vv^2 = \|\vv - \vv^*\|^2 / 2$ when both $\vv, \vv^*$ lie on the unit sphere, in which case our lower bound translates into $\|\vv_T - \vv^*\| \gtrsim \sqrt{d / T}$ for any estimator $\vv_T$ that consumes $T$ samples, which matches the upper bound of Corollary \ref{coro:main} in terms of both $d$ and $T$.

We note that our Corollary \ref{coro:main} and the results of \citet{gao2019stochastic} have different dimension dependency, which is due to a distinct but connected set of assumptions.
We have assumed that each sample $\Xb, \Yb$ follows a vector sub-Gaussian distribution and verifies Assumption \ref{ass:se} required by Proposition \ref{prop_cca_subweibull}, whereas \citet{gao2019stochastic} assume that each coordinate of $\Xb, \Yb$ is sub-Gaussian with a constant parameter.
Hence, the vector sub-Gaussian parameter $\cV$ in our case suffers a dimension-dependent prefactor.

\pb\subsection{Numerical Studies using Synthetic Data}\label{sec:simulation}
\begin{figure}[!tb]
\label{fig:saddle}
  \centering
  \subfigure[]{
    \label{fig:subfig:a} %% label for first subfigure
    \includegraphics[width=2.8in]{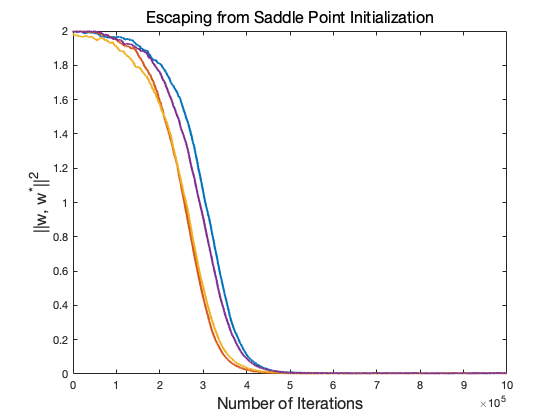}}
  \subfigure[]{
    \label{fig:subfig:b} %% label for second subfigure
    \includegraphics[width=2.8in]{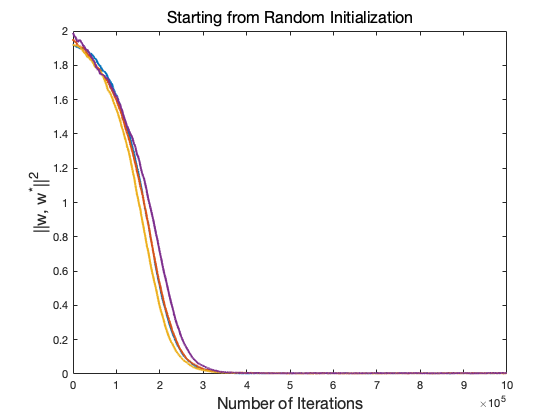}}
      \subfigure[]{
    \label{fig:subfig:c} %% label for second subfigure
    \includegraphics[width=2.8in]{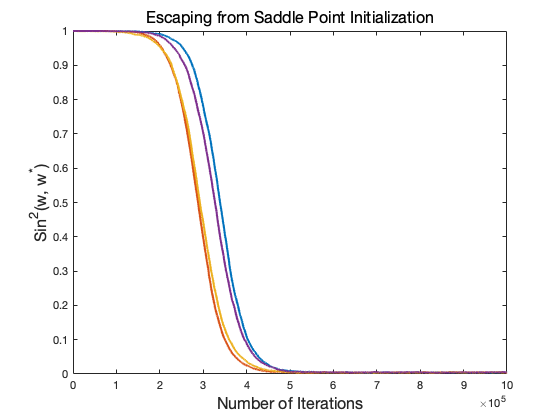}}
      \subfigure[]{
    \label{fig:subfig:d} %% label for second subfigure
    \includegraphics[width=2.8in]{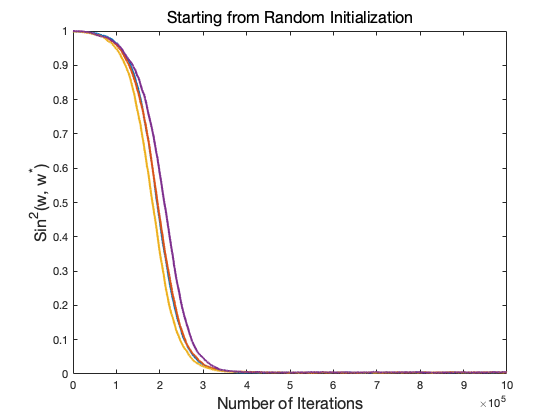}}
  \caption{The comparison between saddle point initialization and random initialization}
  \label{fig:saddle} %% label for entire figure
\end{figure}

In this subsection, we present simulation results for SSGD for the case of rank-one CCA [Algorithm \ref{algo:cca}].
% Recall that in Algorithm \ref{algo:cca} we draw at each iteration two independent pairs of samples, $(\Xb, \Yb)$ and $(\Xb', \Yb')$, and the stochastic oracles $\tilde{\Ab}, \tilde{\Bb}'$ are expressed as
% $$
% \widetilde{\Ab}= \begin{pmatrix} 
% \mathbf{0}		& \Xb \Yb^\top
% \\
% \Yb \Xb^\top	& \mathbf{0}  
% \end{pmatrix}
% \qquad
% \widetilde{\Bb}'= \begin{pmatrix}  
% \Xb' \Xb'^\top	& \mathbf{0}
% \\
% \mathbf{0}		& \Yb' \Yb'^\top  
% \end{pmatrix}.
% $$
The dimensions of the synthetic data samples are picked as $d_1 = 65$ of $\Xb$ and $d_2 = 70$ of $\Yb$.
We generate the covariance matrix for $\Xb, \Yb$ as
\begin{equation}
\bSigma_{\Xb\Xb} = 3\mathbf{I}_{d_1} + \Ab_1
	,\qquad
\bSigma_{\Yb\Yb} = 3\mathbf{I}_{d_2} + \Ab_2
,
\end{equation}
where $\Ab_1, \Ab_2$ are diagonal matrices with each entry along the diagonal obtained as an independent uniform draw from $[0, 1]$. 
To ensure the eigengap of $\bSigma_{\Xb\Xb}^{-\frac{1}{2}}\bSigma_{\Xb\Yb}\bSigma_{\Yb\Yb}^{-\frac{1}{2}}$ is significantly large, in particular, no less than $0.5$, we set
\begin{equation}
\bSigma_{\Xb\Yb}
 = 
\Ab_3 + \bSigma_{\Xb\Xb}^{1/2}\mathbf{U}\diag(0.5, \mathbf{O})\mathbf{V}^\top \bSigma_{\Yb\Yb}^{1/2}
.
\end{equation}
Here $\Ab_3$ is a $d_1\times d_2$ matrix where each entry is generated from an independent $N(0, 1/(d_1+d_2))$ variable with SVD decomposition
$
\bSigma_{\Xb\Xb}^{1/2} \Ab_3 \bSigma_{\Yb\Yb}^{1/2}
    =
\mathbf{U}\mathbf{D}\mathbf{V}^\top
$, and $\mathbf{O}$ is a $(d_1 - 1) \times (d_2 - 1)$ zero matrix.
Note that each step of Algorithm \ref{algo:cca} can be computed in time $\cO(d_1+d_2)$.
Given this setup, we report our numerical findings of Algorithm \ref{algo:cca} as follows:

\paragraph{Saddle-point escaping}
We first discuss the behavior of our algorithm in the presence of saddle points. 
When $\vv_0$ is exactly chosen as a saddle point, we show that SSGD escapes from a plateau of saddle points in the landscape and converges to the local (and global) minimizer. 
For illustrative purposes, the initialization $\vv_0$ is chosen from four saddle points, each of which corresponds to a component of CCA.
We choose the total sample size $T = \text{1e6}$ and set the (constant) step size $\eta = \text{log}(T)/(5T)$.
In Figure~\ref{fig:saddle} we plot the error of the current solution to the optimal solution, where the error is measured both in squared Euclidean distance and in sine-squared.
The first two plots shows the behavior initialized from four different saddle points, and the last two plots shows the behavior initialized from four uniform seeds.
The horizontal axis is the number of iterates and the vertical axis is error $\|\vv_t - \vv^*\|^2$. 

\begin{figure}[!tb]
  \centering
  \subfigure[]{
    \label{fig:linear:a} %% label for first subfigure
    \includegraphics[width=2.8in]{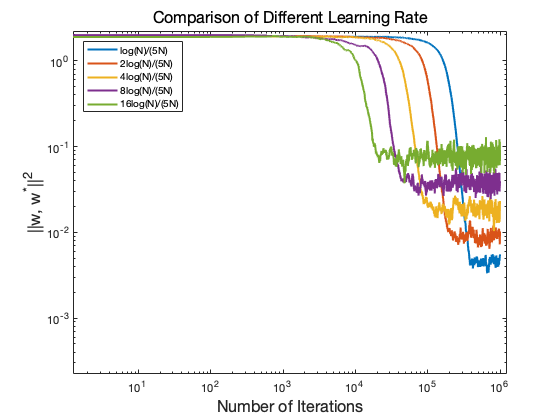}}
  \subfigure[]{
    \label{fig:linear:b} %% label for second subfigure
    \includegraphics[width=2.8in]{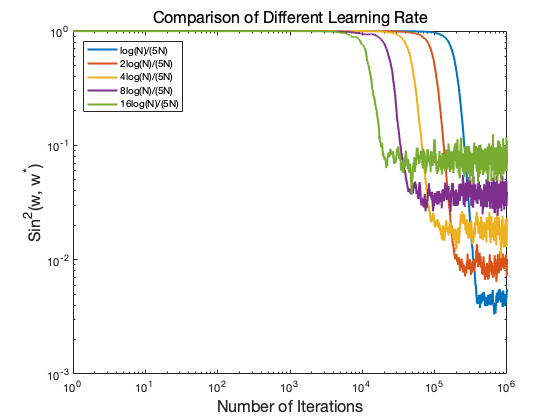}}
\caption{
Log-log plot regarding the convergence with respect to a range of step sizes $\eta$.
Figure \ref{fig:linear:a} illustrates the squared errors in terms of squared distance to optimality $\|\vv - \vv^*\|^2$, and Figure \ref{fig:linear:b} does so in terms of $\sin^2(\vv,\vv^*)$.
}
\label{fig:linear}
\end{figure}
% From Figure~\ref{fig:saddle} the algorithm efficiently escapes from saddle points and the error significantly drops down at a random time, exhibiting a local-global-local three-phase behavior.
% In the Initial Phase, the algorithm gradually escapes from the saddle point; 
% In the Transient Phase, the algorithm quickly  moves towards to the optimum; 
% In the Fluctuation Phase, the algorithm fluctuates around to the optimum.
% For randomly chosen $\vv$, the Initial Phase of random initialization is shorter than the Initial Phase of saddle point initialization.

\paragraph{Relationship between the step size and squared error}
We study the role of step size $\eta$ in our SSGD algorithm.
Set sample size $T = \text{1e6}$ and choose 20 $\eta$'s from 1e--5 to 5e--4 from $
\{\log(T) / (5T)$, $2\log(T) / (5T)$, $4\log(T) / (5T)$, $8\log(T) / (5T)$, $16\log(T) / (5T)\}
$
and plot the squared error $\|\vv - \vv^*\|^2$ on a log-log scale. It is clearly observed from Figure~\ref{fig:linear} that smaller step sizes lead to slower convergence to a stationary point of smaller variance.

We now numerically demonstrate that at stationarity SSGD presents a squared error $\|\vv - \vv^*\|^2$ or $\sin^2(\vv,\vv^*)$ that has a linear relationship with $\eta$.
We compute the averaged squared error of the last 10\% iterates for each run and plot the result in Figure~\ref{fig:linear_2} in a log-log scale.
The horizontal axes of both Figures~\ref{fig:linear_2:c} and~\ref{fig:linear_2:d} represent the step size $\eta$, and the vertical axes of both figures are the squared error $\|\vv-\vv^*\|^2$ and $\sin^2(\vv,\vv^*)$, respectively.
We compute an averaged squared error of the last 10\% iterates for each $\eta$.
Due to ergodicity in the algorithmic final phase, this provides a feasible estimate of its variance around the local (and global) minimizer.
Also, the fitting slope of Figure~\ref{fig:linear_2} provided by the least-square method is 0.9921 (fairly close to 1), which corroborates our theoretical convergence results in Theorems \ref{theo_local} and \ref{coro:saddle}.
These numerical findings are consistent with our theory that the squared error $\|\vv-\vv^*\|^2$ at stationarity has a linear relationship with $\eta$.

\begin{figure}[!tb]
  \centering
      \subfigure[]{
    \label{fig:linear_2:c}
    \includegraphics[width=2.8in]{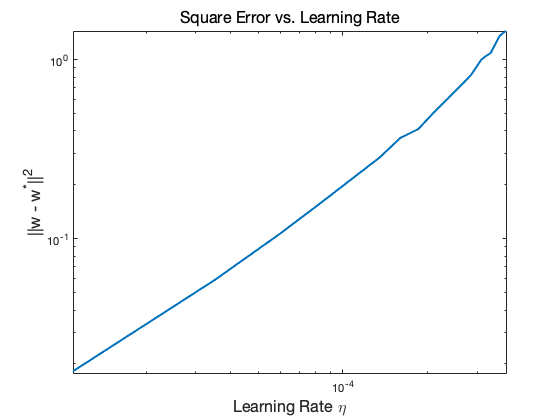}}
      \subfigure[]{
    \label{fig:linear_2:d}
    \includegraphics[width=2.8in]{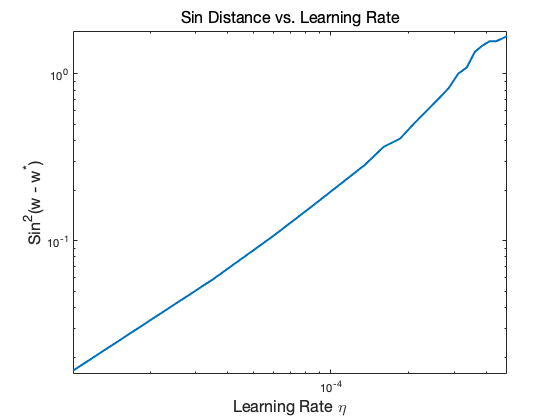}}
  \caption{The relationship between step size $\eta$ and the squared error of our algorithmic estimator to the optimal solution.}
\label{fig:linear_2}
\end{figure}

\pb\section{Proofs}\label{sec_proof-gev}
In this section, we provide detailed proofs of our main results.

\pb\subsection{Proof of Proposition \ref{prop:localconvexity}}\label{sec:localconvexity}
This subsection provides a proof for Proposition \ref{prop:localconvexity} on the convergence to a local minimizer.
Under the initialization condition \eqref{eq:warm}, there exists a local minimizer $\vv^* \in \Bb_{\deltai}(\vv_0)$ of $F(\vv)$ such that $\uu^\top \cH(\vv^*) \uu \ge \alphai \|\uu\|^2$ for all $\uu \in \cT(\vv^*)$.

For a positive quantity $M$ to be determined later, let
\beq\label{cTM}
\cT_M = \inf\left\{
t\ge 1: \|\Gamma(\vv_{t - 1}; \bzeta_t)\| > M
\right\}
.
\eeq
In words, $\cT_M$ is the first $t$ such that the norm of the stochastic scaled-gradient $\Gamma(\vv_{t - 1}; \bzeta_t)$ exceeds $M$.
We first provide the following lemma.

\begin{lemma}\label{lemm:Gtruncate}
Assume all conditions in Theorem \ref{theo_local}.
For any positive $\eps$, let
\beq\label{eq:M}
M
=
\cV \log^{1 / \alpha}\eps^{-1}
.
\eeq
Then, we have
$$
\PP(\cT_M \le T_\eta^*)
\le
2 T_\eta^* \eps
.
$$
\end{lemma}
The proof of Lemma \ref{lemm:Gtruncate} is a straightforward corollary of a union bound and Assumption \ref{ass:se}, and is provided in \S\ref{sec_proof,lemm:Gtruncate}.

Recall the definitions of the manifold gradient $g(\vv)$ and the Hessian $\cH(\vv)$ in \eqref{gvb} and \eqref{Hvb}.
Under a unit spherical constraint $c(\vv) = \|\vv\|^2 - 1 = 0$, their definitions simplify to
\beq\label{eq:gH}
g(\vv)
=
(\Ib - \vv \vv^\top) \nabla F(\vv)
\quad\text{and}\quad
\cH(\vv)
=
\nabla^2 F(\vv) - (\vv^\top \nabla F(\vv)) \Ib
.
\eeq
Taking derivatives, we decompose
\beq\label{eq:nabla_g}
\nabla g(\vv)
=
\cH(\vv) + \cN(\vv)
,
\eeq
where the additional term $\cN(\vv)$ is defined as
\beq\label{eq:N}
\cN(\vv)
=
-\vv (\nabla F(\vv) + \nabla^2 F(\vv) \vv)^\top
.
\eeq
The following lemma shows that $g(\vv), \cH(\vv), \cN(\vv)$ are Lipschitz continuous.
\begin{lemma}\label{lemm:smooth}
Given Assumption \ref{ass:lda}, we have that
$g(\vv), \cH(\vv), \cN(\vv)$ are $L_G, L_H, L_N$-Lipschitz and $\|\cH(\vv)\| \le B_H$ within $\{\vv : \|\vv\| \le 1, \|\vv - \vv^*\| \le \lipr\}$, where the constants are defined as $L_G \equiv L_K + 2 L_F$, $L_H \equiv L_Q + L_F + L_K$, $L_N \equiv L_F + 3L_K + L_Q$, $B_H \equiv L_F + L_K$.
\end{lemma}
A proof of Lemma \ref{lemm:smooth} is deferred to \S\ref{sec_proof,lemm:smooth}.

For notational simplicity, we denote $\cH_* = \cH(\vv^*)$ and $\cN_* = \cN(\vv^*)$, and recall that $\cF_t$ is the filtration generated by $\bzeta_t$.
Then we have the following lemma.
\begin{lemma}\label{lemm:witexpress}
Under Assumptions \ref{ass:lda} and \ref{ass:se}, when $\eta \le 1 / (5M)$, on the event $(\|\Gamma(\vv_{t - 1}; \bzeta_t)\| \le M)$, the update rule \eqref{PSSGD} of $\vv_t$ can be written as
\beq\label{wit}
\vv_t - \vv^*
=
\left(\Ib - \eta D \cH_* - \eta D \cN_*\right) (\vv_{t - 1} - \vv^*)
+ \eta \bxi_t + \eta \bR_t + \eta^2 \bQ_t
,
\eeq
where $\{\bxi_t\}$ forms a vector-valued martingale difference sequence with respect to $\cF_t$, $\bxi_t$ is $\alpha$-sub-Weibull with parameter $G_\alpha \cV$, $\bR_t$ satisfies $\|\bR_t\| \le (D L_H + D L_N + L_D L_G) \|\vv_{t - 1} - \vv^*\|^2$ and $\bQ_t$ satisfies $\|\bQ_t\| \le 7 M^2$.
\end{lemma}
The proof of Lemma \ref{lemm:witexpress} is deferred to \S\ref{sec_proof,lemm:witexpress}.
We define the projection of $\vv_t - \vv^*$ on $\cT(\vv^*)$ as
\beq\label{eq:Delta}
\Delta_t
=
(\Ib - \vv^* {\vv^*}^\top) (\vv_t - \vv^*)
,
\eeq
and the projection of $\cH_*$ on $\cT(\vv^*)$ as
\beq\label{eq:M*}
\cM_*
=
(\Ib - \vv^* {\vv^*}^\top) \cH_* (\Ib - \vv^* {\vv^*}^\top)
.
\eeq
\begin{lemma}\label{lemm:properties}
Under initialization condition \eqref{eq:warm}, the following properties hold:
\begin{enumerate}[label=(\roman*)]
\item
For all $t \ge 0$,
$$
\|(\vv^* {\vv^*}^\top) (\vv_t - \vv^*)\|
=
\frac12 \|\vv_t - \vv^*\|^2
,\qquad
\|\Delta_t\|^2
=
\|\vv_t - \vv^*\|^2 - \frac14 \|\vv_t - \vv^*\|^4
.
$$
If $\vv_t^\top \vv^* \ge 0$, 
\beq\label{eq:Delta_property}
\|\Delta_t\|^2
\le
\|\vv_t - \vv^*\|^2
\le
2 \|\Delta_t\|^2
.
\eeq

\item
When $\eta \le 1 / (D B_H)$, for all $\uu \in \cT(\vv^*)$,
\beq\label{eq:M_property}
\|(\Ib - \eta D \cM_*)^t \Delta_0\|
\le
(1 - \eta D \alphai)^t \|\Delta_0\|
.
\eeq
\end{enumerate}
\end{lemma}
The proof of Lemma \ref{lemm:properties} is deferred to \S\ref{sec_proof,lemm:properties}.
To interpret Lemma \ref{lemm:properties}(i), we denote $\theta \equiv \angle(\vv_t, \vv^*) \in [0, \pi / 2]$, such that $\|\vv_t - \vv^*\| = 2\sin(\theta / 2)$, $\Delta_t = (\Ib - \vv^* \vv^*)^\top (\vv_t - \vv^*) = \sin\theta$, and \eqref{eq:Delta_property} is equivalent to the trigonometric inequality
$$
\sin^2\theta
=
4\sin^2(\theta / 2) \cos^2(\theta / 2)
\le
4\sin^2(\theta / 2)
=
2(1 - \cos\theta)
\le
2(1 - \cos\theta)(1 + \cos\theta)
=
2\sin^2\theta
.
$$
By combining Lemmas \ref{lemm:witexpress} and \ref{lemm:properties}, we have the following lemma for the update rule in terms of $\Delta_t$:

\begin{lemma}\label{lemm:Delta_express}
Under Assumptions \ref{ass:lda}, \ref{ass:se} and initialization condition \eqref{eq:warm}, when $\eta \le 1/(5M)$, on the event $(\|\Gamma(\vv_{t - 1}; \bzeta_t)\| \le M)$, the update \eqref{PSSGD} can be written in terms of $\Delta_t$ as
\beq\label{eq:Delta_update}
\Delta_t
=
\left(\Ib - \eta D \cM_*\right) \Delta_{t - 1}
+
\eta \bchi_t
+
\eta \bS_t
+
\eta^2 \bP_t
;
\eeq
where $\bchi_t, \bS_t, \bP_t \in \cT(\vv^*)$, $\{\bchi_t\}$ forms a vector-valued martingale difference sequence with respect to $\cF_t$, $\bchi_t$ is $\alpha$-sub-Weibull with parameter $G_\alpha \cV$, $\bS_t$ satisfies $\|\bS_t\| \le \rhoi \|\vv_{t - 1} - \vv^*\|^2$ and $\bP_t$ satisfies $\|\bP_t\| \le 7 M^2$.
\end{lemma}
Proof of Lemma \ref{lemm:Delta_express} is deferred to \S\ref{sec_proof,lemm:Delta_express}.
Here we have $\rhoi = D (L_H + L_N + B_H / 2) + L_D L_G$, which is consistent with its definition in \eqref{rhoi}.

Now, to analyze the iteration $\Delta_t$ we need to control its tail behavior.
We define the truncated version
\beq\label{barRnk}
\tbS_t
=
\bS_t 1_{(\cT_M > t)}
,\qquad
\tbP_t
=
\bP_t 1_{(\cT_M > t)}
,
\eeq
let $\tDelta_0 = \Delta_0$, and define the coupled process iteratively
\beq\label{geomUhat}
\tDelta_t
=
\left(\Ib - \eta D \cM_* \right) \tDelta_{t - 1}
+
\eta \bchi_t
+
\eta \tbS_t
+
\eta^2 \tbP_t
.
\eeq
The $\tDelta_t$ iteration avoids the potential issues of summation over $\bP_t$. 
We conclude the following lemma that characterizes the coupling relation $\tDelta_t = \Delta_t$, which allows us to analyze the coupled iteration $\tDelta_t$.

\begin{lemma}\label{lemm:coupling1}
For each $t \ge 0$ we have $\tDelta_t = \Delta_t$ on the event $(\cT_M > t)$.
Furthermore, we have for all $t \ge 1$
\beq\label{grad}
\begin{aligned}
\tDelta_t
&=
\left(\Ib - \eta D \cM_*\right)^t \Delta_0
 +
\eta \sum_{s = 1}^t \left(\Ib - \eta D \cM_*\right)^{t - s} \bchi_s \notag\\
&\quad
+
\eta \sum_{s = 1}^t \left(\Ib - \eta D \cM_*\right)^{t - s} \tbS_s
+
\eta^2 \sum_{s = 1}^t \left(\Ib - \eta D \cM_*\right)^{t - s} \tbP_s
.
\end{aligned}\eeq
\end{lemma}
We defer the proof of Lemma \ref{lemm:coupling1} in \S\ref{sec_proof,lemm:coupling1}.

Next we provide a lemma that tightly characterizes the approximations in \eqref{grad} that $\tDelta_t \approx (\Ib - \eta D \cM_*)^t \Delta_0$.

\begin{lemma}\label{lemm:coord}
Let $\eta \le \min\left\{1 / (D B_H), 1 / (5M)\right\}$ and $T \ge 1$.
Then with probability at least
$$
1 - \left(
12 + 8 \left(\frac{3}{\alpha}\right)^{\frac{2}{\alpha}} \log^{- \frac{\alpha + 2}{\alpha}} \eps^{-1}
\right) T \eps
,
$$
the algorithm satisfies for each $t \in [0, T]$, conditioning on $\|\vv_s - \vv^*\| \le \radius$ for all $s = 0,\dots, t - 1$ for some $\radius > 0$
\begin{align}\label{xt_concentration}
\left\|
\tDelta_t - (\Ib - \eta D \cM_*)^t \Delta_0
\right\|
\le
\frac{8G_\alpha \cV}{\sqrt{D \alphai}} \log^{\frac{\alpha + 2}{2\alpha}} \eps^{-1} \cdot \eta^{1 / 2}
+
\frac{\rhoi \radius^2}{D \alphai}
+
\frac{7 \cV^2}{D \alphai} \log^{\frac{2}{\alpha}}\eps^{-1} \cdot \eta
.
\end{align}
\end{lemma}
The proof of Lemma \ref{lemm:coord} is provided in \S\ref{sec_proof,lemm:coord}.

In the following lemma we prove that when the initial iterate $\vv_0$ is sufficiently close to the minimizer $\vv^*$ and $\radius$ is appropriately chosen to be dependent on $\Delta_0$ and $\widetilde{\Theta}(\eta^{1 / 2})$, the conditioning event occurs almost surely on a high-probability event.

\begin{lemma}\label{lemm:radius}
When initialization 
$$
\|\Delta_0\|
\le
\left\{
\frac{D \alphai}{2^5 G_\alpha \rhoi}
,
\lipr
\right\}
,
$$
for any positives $\eta, \eps$ satisfying scaling condition \eqref{eq:eta_scaling}, with probability at least
$$
1 - \left(
14 + 8 \left(\frac{3}{\alpha}\right)^{\frac{2}{\alpha}} \log^{- \frac{\alpha + 2}{\alpha}} \eps^{-1}
\right) T \eps
,
$$
for all $t \in [0, T]$ we have
$$
\|\Delta_t\|
\le
2 \max\left\{
\|\Delta_0\|
,~
\frac{2^7 G_\alpha \cV}{\sqrt{D \alphai}} \log^{\frac{\alpha + 2}{2\alpha}} \eps^{-1} \cdot \eta^{1 / 2}
\right\}
,
$$
and if $T_\eta^* \in [0, T]$, at time $T_\eta^*$ we have
$$
\|\Delta_{T_\eta^*}\|
\le
\frac12 \max\left\{
\|\Delta_0\|
,~
\frac{2^7 G_\alpha \cV}{\sqrt{D \alphai}} \log^{\frac{\alpha + 2}{2\alpha}} \eps^{-1} \cdot \eta^{1 / 2}
\right\}
.
$$
\end{lemma}
Lemma \ref{lemm:radius}, whose proof is given in \S\ref{sec_proof,lemm:radius}, implies that the iteration keeps $\|\Delta_t\| \le 2\|\Delta_0\|$ unless $\vv$ is within a noisy neighborhood of the local minimizer $\vv^*$, where we recall the definition of $\Delta_t$ in \eqref{eq:Delta}.

\pb
Finally, Proposition \ref{prop:localconvexity} is proved by combining Lemmas \ref{lemm:properties} and \ref{lemm:radius}.

\pb\subsection{Proof of Theorems \ref{theo:saddle} and \ref{coro:saddle}}\label{sec:escapingSSGD}

In this subsection, we aim to prove Theorem \ref{theo:saddle}.
To deal with points with strong gradient corresponding to (i) in Definition \ref{defi_strictsaddle}, we use the following lemma that is adapted from \citet[Lemma 38]{ge2015escaping}.

\begin{proposition}\label{prop:gradientGe}
Assume all conditions in Theorem \ref{theo:saddle} as well as $\sqrt{2d\cV^2 L_G D_+\eta} < \betai$, we have on the event $\left(\|\nabla F(\vv_t)\| \ge \sqrt{2d\cV^2 L_G D_+\eta} \right)$ that
\beq\label{gradientGe}
\Exs\left[ F(\vv_{t+1}) - F(\vv_t) \mid \cF_t \right]
\le
- 0.5 d \sigma^2 L_G D_-^2 \eta^2
.
\eeq
\end{proposition}

A core problem involves escaping from saddle points that corresponds to (iii) in Definition \ref{defi_strictsaddle}, we conclude the following modification from \citet[Lemma 40]{ge2015escaping}.

\begin{proposition}\label{prop:saddleGe}
Assume all conditions in Theorem \ref{theo:saddle} as well as $\sqrt{2 \eta \sigma^2 L_G d D_+} < \betai$.
Then on the event
$$
\left\{
\|\nabla F(\vv_0)\| < \sqrt{2d\cV^2 L_G D_+\eta}
,\
\lambda_{\min}(\cH(\vv_0)) \le -\gammai
\right\},
$$
there is a stopping time $\cT(\vv_0) \le T_{\max}$ almost surely such that
\beq\label{saddleGe}
\Exs F(\vv_{\cT(\vv_0)}) - F(\vv_0)
 \le
- 0.5 \sigma^2 D_- \eta
,
\eeq
where $T_{\max}$ is fixed and independent of $\vv_0$ defined as 
$$
T_{\max}
 = 
0.5\gammai^{-1} D_-^{-1} \eta^{-1} \log \left(
 \frac{6d\cV}{\sigma}
\right)
.
$$
\end{proposition} 
Proofs of Propositions \ref{prop:gradientGe} and \ref{prop:saddleGe} are straightforward generalization of relevant proofs of \citep{ge2015escaping}, and hence we omit the details.

\begin{proof}[Proof of Theorem \ref{theo:saddle}]
While this proof can be done in a similar fashion as Theorem 36 in \citet{ge2015escaping}, here we provide a different proof using stopping-time techniques.

\begin{enumerate}[label=(\roman*)]
\item
Given \eqref{etamax}, we split the state space $\cS^{d-1}$ into three distinct regions: let
$$
\cQ_1
 = 
\left\{\vv\in \cS^{d-1}:
 \|\nabla F(\vv)\| \ge \sqrt{2 d\cV^2 L_G D_+\eta}
\right\}
,
$$
and let
$$
\cQ_2
 = 
\left\{\vv\in \cS^{d-1}:
 \|\nabla F(\vv)\| < \sqrt{2 d\cV^2 L_G D_+\eta}
,\
\lambda_{\min} (\cH(\vv) ) \le -\gammai
\right\}
.
$$
Define a stochastic process $\{\cT_i\}$ s.t.~$\cT_0 = 0$, and 
\beq\label{taudef}
\cT_{i+1} = 
\cT_i
 +
1_{\cQ_1} (\vv_{\cT_i})
 +
\cT(\vv_{\cT_i}) 1_{\cQ_2} (\vv_{\cT_i})
,
\eeq
where $\cT(\vv_{\cT_i}) \le T_{\max}$ is defined in Proposition \ref{prop:saddleGe}.
By \eqref{gradientGe} in Lemma \ref{prop:gradientGe} and \eqref{saddleGe} in Proposition \ref{prop:saddleGe}, we know that on $(\vv_{\cT_i} \in \cQ_1)$
$$
\Exs [F(\vv_{\cT_{i+1}}) - F(\vv_{\cT_i})
\mid
\cF_{\cT_i} ] 
\le
- 0.5 d \sigma^2 L_G D_-^2 \eta^2
,
$$
and on $(\vv_{\cT_i} \in \cQ_2)$
$$
\Exs [F(\vv_{\cT_{i+1}}) - F(\vv_{\cT_i})
\mid
\cF_{\cT_i} ] 
\le
- 0.5 \sigma^2 D_- \eta
.
$$
Combining the above two displays and \eqref{taudef}, we have
\beq\label{fdecrease}
\begin{split}
 & \quad
\Exs [F(\vv_{\cT_{i+1}}) - F(\vv_{\cT_i})\mid\cF_{\cT_i} ] \\
&\le 
 - \min\left(
 0.5 d \sigma^2 L_G D_-^2 \eta^2
 ,
 \frac{0.5 \sigma^2 D_- \eta}{
 0.5\gammai^{-1} D_-^{-1} \eta^{-1} \log \left(
 \frac{6d\cV}{\sigma} \right)
}
 \right)
 \cdot \Exs \left[ \cT_{i+1} - \cT_i \mid \cF_{\cT_i} \right]
\\
 &\le 
 - \min\left(
 0.5 d L_G
 ,
 \gammai 
\log^{-1} \left( \frac{6d\cV}{\sigma} \right)
 \right)
 \sigma^2 D_-^2 \eta^2
 \cdot \Exs \left[ \cT_{i+1} - \cT_i \mid \cF_{\cT_i} \right],
\end{split}
\eeq
\text{
on $\{\vv_{\cT_i} \in \cQ_1\cup \cQ_2\}$.}

\item
Let $\cI \in [0,\infty]$ be the (random) first index $i$ such that $\vv_{\cT_i} \in (\cQ_1 \cup \cQ_2)^c$. 
We conclude immediately that $(\cI > i) \in \cF_{\cT_i}$, and $(\cI > i) \subseteq \left(\vv_{\cT_i} \in \cQ_1\cup\cQ_2\right)$.
Applying \eqref{fdecrease} gives
\begin{align*}
\Exs\left[
F(\vv_{\cT_{\cI}}) - F(\vv_0)
\right]
 &=
\Exs\left[
\sum_{i=0}^\infty \left( F(\vv_{\cT_{i+1}}) - F(\vv_{\cT_i}) \right)1_{\cI > i}
\right]
 \\&\le 
 - \min\left(
 0.5 d L_G
 ,
 \gammai 
\log^{-1} \left( \frac{6d\cV}{\sigma} \right)
 \right)
 \sigma^2 D_-^2 \eta^2
 \cdot \Exs \cT_{\cI}
 \\&\le 
 - \min\left(
 0.5 d L_G
 ,
 \gammai 
\log^{-1} \left( \frac{6d\cV}{\sigma} \right)
 \right)
 \sigma^2 D_-^2 \eta^2
 \cdot T
 \cdot \PP\left( \cT_{\cI} \ge T \right)
,
\end{align*}
where $T \ge 0$ is any constant.
Plugging in $T = T_1$ as in \eqref{T1} gives
\begin{align*}
\PP\left(
 \cT_{\cI} \ge T_1
 \right) 
&\le
\frac{\Exs\left[
F(\vv_0) - F(\vv_{\cT_{\cI}})
\right]}
{
 \min\left(
 0.5 d L_G
 ,
 \gammai 
\log^{-1} \left( \frac{6d\cV}{\sigma} \right)
 \right)
 \sigma^2 D_-^2 \eta^2
\cdot
T_1
}
\\&\le
\frac{2\|F\|_\infty}{ 4\|F\|_\infty}
=
\frac12
.
\end{align*}
In words, event $(\cT_{\cI} < T_1)$ has at least $1/2$ probability, on which the iteration $\vv_t$ must enter $\left( \cQ_1\cup\cQ_2 \right)^c$ by time $T_1$ at least once.

\item
Noting that the argument above holds for all initial points $\vv_0\in \cQ_1\cup\cQ_2$, so one can use Markov property and conclude that within $T_1 \cdot \lceil \log_2 (\kappa^{-1}) \rceil$ steps where $T_1$ was defined in \eqref{T1}, iteration $\{\vv_t\}$ must enter $\left(\cQ_1\cup\cQ_2\right)^c$ at least once with probability at least $1-\kappa$.
The rest of our proof follows from the definition of strict-saddle function.

\end{enumerate}

\end{proof}

\begin{proof}[Proof of Theorem \ref{coro:saddle}]
The conclusion is reached by directly combining Theorems \ref{theo_local} and \ref{theo:saddle}, setting $\cA_T = \cH_{\ref{theo_local}}$, along with an application of strong Markov property.
\end{proof}

\pb\subsection{Proof of Proposition \ref{prop_gev_lipschitz}}\label{sec_proof,prop_gev_lipschitz}

\begin{proof}[Proof of Proposition \ref{prop_gev_lipschitz}]
For the GEV problem setting, the gradient and the Hessian of the objective function $F(\vv)$ are
$$
\nabla F(\vv)
=
-2\frac{(\vv^\top \Bb \vv) \Ab \vv - (\vv^\top \Ab \vv) \Bb \vv}{(\vv^\top \Bb \vv)^2}
,
$$
$$
\nabla^2  F(\vv)
=
-2\frac{
(\vv^\top \Bb \vv) \Ab - (\vv^\top \Ab \vv) \Bb
+
2( \Ab \vv \vv^\top \Bb - \Bb \vv \vv^\top \Ab )}{(\vv^\top \Bb \vv)^2
}
+
8\frac{\left[(\vv^\top \Bb \vv) \Ab - (\vv^\top \Ab \vv) \Bb
\right] \vv \vv^\top \Bb
}{(\vv^\top \Bb \vv)^3}
.
$$
We first notice that, for $\vv \in \{\vv: \|\vv\| \le 1, \|\vv - \vv^*\| \le \lipr\}$,
\begin{align*}
\|\nabla D(\vv)\|
=
\left\|
2 (\vv^\top \Bb \vv) \Bb \vv
\right\|
\le
2 \|\Bb\|^2
,
\end{align*}
which indicates that $D(\vv)$ has Lipschitz constant $L_D \equiv 2 \|\Bb\|^2$.
Secondly, we introduce an arbitrary unit vector $\ww$ and take derivative of vector $\nabla^2 F(\vv) \ww$ w.r.t. $\vv$ as
\begin{align*}
\nabla_{\vv} \left[
\nabla^2 F(\vv) \ww
\right]
&=
-2 \frac{2 \Ab \ww \vv^\top \Bb - 2 \Bb \vv \vv^\top \Ab + 2 (\vv^\top \Bb \ww) \Ab + 2 \Ab \vv \ww^\top \Bb - 2 (\vv^\top \Ab \ww) \Bb - 2 \Bb \vv \ww^\top \Ab}{(\vv^\top \Bb \vv)^2}
\\&\qquad+
8 \frac{\left[ (\vv^\top \Bb \vv) \Ab - (\vv^\top \Ab \vv) \Bb + 2( \Ab \vv \vv^\top \Bb - \Bb \vv \vv^\top \Ab ) \right] \ww \vv^\top \Bb}{(\vv^\top \Bb \vv)^3}
\\&\qquad+
8\frac{ \left[(\vv^\top \Bb \vv) \Ab - (\vv^\top \Ab \vv) \Bb
\right] \vv \ww^\top \Bb
}{(\vv^\top \Bb \vv)^3}
\\&\qquad+
8\frac{(\vv^\top \Bb \ww) \left[
(\vv^\top \Bb \vv) \Ab - (\vv^\top \Ab \vv) \Bb
+
2( \Ab \vv \vv^\top \Bb - \Bb \vv \vv^\top \Ab )
\right]
}{(\vv^\top \Bb \vv)^3}
\\&\qquad
-48
\left[
\frac{ (\vv^\top \Bb \ww)\left[(\vv^\top \Bb \vv) \Ab - (\vv^\top \Ab \vv) \Bb
\right] \vv
\vv^\top \Bb}{(\vv^\top \Bb \vv)^4}
\right]
.
\end{align*}
The five terms on the right-hand side have norm bounded by $\frac{24 \|\Ab\| \|\Bb\|}{(1 - \lipr)^2 \lambda_{\min}^2(\Bb)}$, $\frac{48 \|\Ab\| \|\Bb\|^2}{(1 - \lipr)^3 \lambda_{\min}^3(\Bb)}$, $\frac{16 \|\Ab\| \|\Bb\|^2}{(1 - \lipr)^3 \lambda_{\min}^3(\Bb)}$, $\frac{48 \|\Ab\| \|\Bb\|^2}{(1 - \lipr)^3 \lambda_{\min}^3(\Bb)}$, $\frac{96 \|\Ab\| \|\Bb\|^3}{(1 - \lipr)^4 \lambda_{\min}^4(\Bb)}$ respectively, which implies that
$$
\left\| \nabla_{\vv} \left[
\nabla^2 F(\vv) \ww
\right] \right\|
\le
\frac{232 \|\Ab\| \|\Bb\|^3}{\lambda_{\min}^4(\Bb)}
.
$$
Therefore, for all $\vv_1, \vv_2 \in \{\vv : \|\vv\| \le 1, \|\vv - \vv^*\| \le \lipr\}$, we have
$$
\left\|
\nabla^2 F(\vv_1) - \nabla^2 F(\vv_2)
\right\|
=
\max_{\|\ww\| = 1} \left\|
\nabla^2 F(\vv_1) \ww - \nabla^2 F(\vv_2) \ww
\right\|
\le
\frac{232 \|\Ab\| \|\Bb\|^3}{(1 - \lipr)^4 \lambda_{\min}^4(\Bb)} \|\vv_1 - \vv_2\|
,
$$
indicating $\nabla^2 F(\vv)$ has Lipschitz constant $L_Q \equiv \frac{232 \|\Ab\| \|\Bb\|^3}{(1 - \lipr)^4 \lambda_{\min}^4(\Bb)}$.

Similarly, we also notice for all $\vv \in \{\vv : \|\vv\| \le 1, \|\vv - \vv^*\| \le \lipr\}$,
$$
\left\|
\nabla F(\vv)
\right\|
\le
\frac{4 \|\Ab\| \|\Bb\|}{(1 - \lipr)^2 \lambda_{\min}^2(\Bb)}
,\qquad
\left\|
\nabla^2 F(\vv)
\right\|
\le
\frac{28 \|\Ab\| \|\Bb\|^2}{(1 - \lipr)^3 \lambda_{\min}^3(\Bb)}
,
$$
which indicates that $F(\vv)$ has Lipschitz constant $L_F \equiv \frac{4 \|\Ab\| \|\Bb\|}{(1 - \lipr)^2 \lambda_{\min}^2(\Bb)}$ and $\nabla F(\vv)$ has Lipschitz constant $L_K \equiv \frac{28 \|\Ab\| \|\Bb\|^2}{(1 - \lipr)^3 \lambda_{\min}^3(\Bb)}$.
\end{proof}

\pb\subsection{Proof of Theorem \ref{theo:asympnorm}}\label{sec_proof,theo:asympnorm}
To prove Theorem \ref{theo:asympnorm}, we first present the following Lemma \ref{lemm:representation} on a linear representation of $\cM_* (\ivv_T^{(\eta)} - \vv^*)$.

\begin{lemma}[Representation Lemma]\label{lemm:representation}
Under Assumptions \ref{ass:lda}, \ref{ass:se} and given initialization condition \eqref{eq:warm}, for any $T \ge K_{\eta, \epsilon} T_\eta^*$ and positive constants $\eta, \eps$ satisfying the scaling condition
$$
5\cV \log^{1 / \alpha}\eps^{-1} \cdot \eta
\le
1
,
$$
we have
\beq\label{eq:rep}\begin{split}
\cM_* \left(\ivv_T^{(\eta)} - \vv^*\right)
&=
\frac{1}{D (T - K_{\eta, \epsilon} T_\eta^*)} \sum_{t = K_{\eta, \epsilon} T_\eta^* + 1}^T \bchi_{t + 1}
+
\frac{1}{D (T - K_{\eta, \epsilon} T_\eta^*)} \sum_{t = K_{\eta, \epsilon} T_\eta^* + 1}^T \bS_{t + 1}
\\&\qquad+
\frac{\eta}{D (T - K_{\eta, \epsilon} T_\eta^*)} \sum_{t = K_{\eta, \epsilon} T_\eta^* + 1}^T \bP_{t + 1}
+
\frac{1}{D (T - K_{\eta, \epsilon} T_\eta^*) \eta} (\Delta_{K_{\eta, \epsilon} T_\eta^* + 1} - \Delta_{T + 1})
,
\end{split}\eeq
where $\bchi_t, \bS_t, \bP_t$ are vectors in the tangent space $\cT(\vv^*)$.
Here $\bchi_t$ is defined as
\beq\label{eq:lemm-bxi}
\bchi_t
\equiv
(\Ib - \vv^* {\vv^*}^\top) (\Gamma(\vv_{t - 1}; \bzeta_t) - D(\vv_{t - 1}) \nabla F(\vv_{t - 1}))
,
\eeq
which is $\alpha$-sub-Weibull with parameter $G_\alpha \cV$. 
The sequence $\{\bchi_t\}$ forms a vector-valued martingale difference sequence with respect to $\cF_t$.
$\bS_t$ satisfies $\|\bS_t\| \le \rhoi \|\vv_{t - 1} - \vv^*\|^2$.
On the event $\cH_{\ref{theo_local}}$ defined in Theorem \ref{theo_local}, using a total sample size $T + 1$, each $\bP_t$ satisfies $\|\bP_t\| \le 7 \cV^2 \log^{2 / \alpha}\eps^{-1}$.
\end{lemma}

\begin{proof}[Proof of Lemma \ref{lemm:representation}]
Telescoping \eqref{eq:Delta_update} in Lemma \ref{lemm:Delta_express} for $t = K_{\eta, \epsilon} T_\eta^* + 2, \ldots, T + 1$ gives
$$\begin{aligned}
\eta D \cM_* \sum_{t = K_{\eta, \epsilon} T_\eta^* + 1}^T \Delta_t
	&=
(\Delta_{K_{\eta, \epsilon} T_\eta^* + 1} - \Delta_{T + 1})
+
\eta \sum_{t = K_{\eta, \epsilon} T_\eta^* + 1}^T \bchi_{t + 1}
	\\&\quad
+
\eta \sum_{t = K_{\eta, \epsilon} T_\eta^* + 1}^T \bS_{t + 1}
+
\eta^2 \sum_{t = K_{\eta, \epsilon} T_\eta^* + 1}^T \bP_{t + 1}
.
\end{aligned}$$
Plugging in the definitions of $\Delta_t, \ivv_T^{(\eta)}$ in \eqref{eq:Delta}, \eqref{eq:ave} gives \eqref{eq:rep}.
For event $\cH_{\ref{theo_local}}$ defined in Theorem \ref{theo_local} using total sample size $T + 1$, the proof of Lemma \ref{lemm:radius} in \S\ref{sec_proof,lemm:radius} shows that $\cH_{\ref{theo_local}} \subseteq \big\{\|\Gamma(\vv_{t - 1}; \bzeta_t)\| \le M : 1 \le t \le T + 2\big\}$.
The rest of Lemma \ref{lemm:representation} directly follows Lemma \ref{lemm:Delta_express}.
\end{proof}

With Lemma \ref{lemm:representation} in hand, we are ready to prove Theorem \ref{theo:asympnorm}.

\begin{proof}[Proof of Theorem \ref{theo:asympnorm}]
For a given $T$, we apply Theorem \ref{theo_local} and Lemma \ref{lemm:representation} with $\epsilon = 1 / T^2$, such that $\PP(\cH_{\ref{theo_local}}) \rightarrow 1$ and the scaling condition \eqref{eq:eta_scaling} is satisfied under condition \eqref{eq:condition}.
Using a coupling approach we can safely ignore the small probability event and concentrate on the event $\cH_{\ref{theo_local}}$, where we have
$$
\left\|
\frac{1}{D (T - K_{\eta, \epsilon} T_\eta^*)} \sum_{t = K_{\eta, \epsilon} T_\eta^* + 1}^T \bS_{t + 1}
\right\|
\le
\frac{2^{\frac{\alpha + 2}{\alpha} + 17} \rhoi G_\alpha^2 \cV^2}{D^2 \alphai} \eta \log^{\frac{\alpha + 2}{\alpha}} T
,
$$
$$
\left\|
\frac{\eta}{D (T - K_{\eta, \epsilon} T_\eta^*)} \sum_{t = K_{\eta, \epsilon} T_\eta^* + 1}^T \bP_{t + 1}
\right\|
\le
\frac{7 \cdot 2^{\frac{2}{\alpha}} \cV^2}{D} \eta \log^{\frac{2}{\alpha}} T
.
$$
Using the relation $\|\Delta_t\| \le \|\vv_t - \vv^*\| \le \sqrt{2} \|\Delta_t\|$, given in Proposition \ref{prop:localconvexity}, and applying Theorem \ref{theo_local} on event $\cH_{\ref{theo_local}}$ we also have
$$
\left\|
\frac{1}{D (T - K_{\eta, \epsilon} T_\eta^*) \eta} (\Delta_{K_{\eta, \epsilon} T_\eta^* + 1} - \Delta_{T + 1})
\right\|
\le
\frac{2^{\frac{\alpha + 2}{2\alpha} + \frac{17}{2} + 1}G_\alpha \cV}{\sqrt{D^3 \alphai}} \frac{\log^{\frac{\alpha + 2}{2\alpha}} T}{(T - K_{\eta, \epsilon} T_\eta^*) \eta^{1 / 2}}
.
$$
Under condition \eqref{eq:condition}, as $T \rightarrow \infty, \eta \rightarrow 0$, we have the following almost-sure convergences
$$
\frac{\sqrt{T}}{D (T - K_{\eta, \epsilon} T_\eta^*)} \sum_{t = K_{\eta, \epsilon} T_\eta^* + 1}^T \bS_{t + 1}
\rightarrow
\bm{0}
\quad\text{a.s.}
$$
$$
\frac{\eta \sqrt{T}}{D (T - K_{\eta, \epsilon} T_\eta^*)} \sum_{t = K_{\eta, \epsilon} T_\eta^* + 1}^T \bP_{t + 1}
\rightarrow
\bm{0}
\quad\text{a.s.}
$$
$$
\frac{\sqrt{T}}{D (T - K_{\eta, \epsilon} T_\eta^*) \eta} (\Delta_{K_{\eta, \epsilon} T_\eta^* + 1} - \Delta_{T + 1})
\rightarrow
\bm{0}
\quad\text{a.s.}
$$
From \eqref{eq:Sigma} and \eqref{eq:lemm-bxi}, the covariance matrix of $\bxi_t$---i.e., the projection of scaled-gradient noise onto the tangent space $\cT(\vv^*)$---can be denoted by
$$
\bPhi(\vv_{t - 1})
\equiv
(\Ib - \vv^* {\vv^*}^\top) \bSigma(\vv_{t - 1}) (\Ib - \vv^* {\vv^*}^\top)
.
$$
We denote the covariance matrix at local minimizer $\vv^*$ as $\bPhi_* \equiv \bPhi(\vv^*) = (\Ib - \vv^* {\vv^*}^\top) \bSigma_* (\Ib - \vv^* {\vv^*}^\top)$.
Using the central limit theorem and the Slutsky theorem, we have the following convergence-in-distribution result under the condition \eqref{eq:condition} as $T \rightarrow \infty, \eta \rightarrow 0$:
$$
\frac{1}{\sqrt{T}} \sum_{t = K_{\eta, \epsilon} T_\eta^* + 1}^T \bchi_{t + 1}
\stackrel{d}{\rightarrow} 
N\left(\mathbf{0}, \bPhi_*\right)
.
$$
Combining these results with \eqref{eq:rep} in Lemma \ref{lemm:representation}, under condition \eqref{eq:condition}, as $T \rightarrow \infty, \eta \rightarrow 0$ we have convergence in distribution:
\beq\label{eq:project_an}
\sqrt{T} \cM_* \left(\ivv_T^{(\eta)} - \vv^*\right)
	\stackrel{d}{\rightarrow} 
\mathcal{N}(\mathbf{0}, D^{-2} \cdot \bPhi_*)
.
\eeq
Since $\cM_*^- \cM_* = \Ib - \vv^* {\vv^*}^\top$ and $\cM_*^- \bPhi_* \cM_*^- = \cM_*^- \bSigma_* \cM_*^-$, \eqref{eq:project_an} is equivalent to
\begin{align}\label{eq:project_an2}
\sqrt{T} (\Ib - \vv^* {\vv^*}^\top) \left(
\ivv_T^{(\eta)} - \vv^*
\right)
\stackrel{d}{\rightarrow} 
N\left(\mathbf{0},
D^{-2} \cdot \cM_*^- \bSigma_* \cM_*^-
\right)
,
\end{align}
which omits the asymptotic analysis in the direction parallel to $\vv^*$.
To study the asymptotic property of $\vv^* {\vv^*}^\top (\ivv_T^{(\eta)} - \vv^*)$, we first notice that in Lemma \ref{lemm:properties} in \S\ref{sec:localconvexity} we know that for all $\vv \in \real^d$ with $\|\vv\| = 1$, $\| \vv^* {\vv^*}^\top (\vv - \vv^*) \| = 1 - {\vv^*}^\top \vv = \frac12 \| \vv - \vv^* \|^2$.

Applying Theorem \ref{theo_local}, on event $\cH_{\ref{theo_local}}$ we have:
$$\begin{aligned}
&\quad
\left\|
\sqrt{T} \cdot \vv^* {\vv^*}^\top (\ivv_T^{(\eta)} - \vv^*)
\right\|
=
\frac{1}{2\sqrt{T}} \sum_{t = K_{\eta, \epsilon} T_\eta^* + 1}^T \left\|
\vv_t - \vv^*
\right\|^2
\\&\le
\frac{2^{\frac{\alpha + 2}{\alpha} + 17}G_\alpha^2 \cV^2}{D \alphai}
\cdot
\frac{\eta (T - K_{\eta, \epsilon} T_\eta^*) \log^{\frac{\alpha + 2}{\alpha}} T}{\sqrt{T}}
\lesssim
\sqrt{\eta^2 T \log^{\frac{2\alpha + 4}{\alpha}} T } \to 0
,
\end{aligned}$$
where in the second line we used the first condition in \eqref{eq:condition}.
Under condition \eqref{eq:condition}, as $T \rightarrow \infty, \eta \rightarrow 0$, we have almost-sure convergence
\beq\label{eq:vv*_0}
\sqrt{T} \cdot \vv^* {\vv^*}^\top \left(
\ivv_T^{(\eta)} - \vv^*
\right)
\rightarrow
\bm{0}
\quad\text{a.s.}
\eeq
Adding up \eqref{eq:project_an2} and \eqref{eq:vv*_0} and applying the Slutsky theorem, we conclude $\eqref{eq:an}$ and Theorem \ref{theo:asympnorm}.
\end{proof}

\pb\subsection{Proof of Proposition \ref{prop_cca_subweibull}}\label{sec_proof,prop_cca_subweibull}

\begin{proof}[Proof of Proposition \ref{prop_cca_subweibull}]
For notational simplicity, we denote vector $\vv \in \real^{d_x + d_y}$ as $\vv^\top = (\vv_x^\top, \vv_y^\top)$ for $\vv_x \in \real^{d_x}, \vv_y \in \real^{d_y}$.
For any vectors $\ww_1, \ww_2 \in \real^{d_x}$ with $\|\ww_1\| \le 1, \|\ww_2\| \le 1$, using Lemma \ref{lemm:Orlicz_norm_property} we have
$$
\left\|
\ww_1^\top \Xb \Xb^\top \ww_2
\right\|_{\psi_1}
\le
\left\|
\ww_1^\top \Xb
\right\|_{\psi_2} \left\|
\ww_2^\top \Xb
\right\|_{\psi_2}
\le
\cV_x^2
,
$$
which indicates that
$$
\left\|
\vv_x^\top  \Xb \Xb^\top \vv_x
\right\|_{\psi_1}
\le
\cV_x^2
,\qquad
\left\|
\Xb \Xb^\top \vv_x
\right\|_{\psi_1}
\le
\cV_x^2
.
$$
Similarly, we can show
$$
\left\|
\vv_y^\top  \Yb \Yb^\top \vv_y
\right\|_{\psi_1}
\le
\cV_y^2
,\qquad
\left\|
\Yb \Yb^\top \vv_y
\right\|_{\psi_1}
\le
\cV_y^2
,
$$
and
$$
\left\| \vv_x^\top  \Xb \Yb^\top \vv_y \right\|_{\psi_1}
	\le
\cV_x \cV_y
	,\qquad
\left\| \Xb \Yb^\top \vv_y \right\|_{\psi_1}
	\le
\cV_x \cV_y
	,\qquad
\left\| \Yb \Xb^\top \vv_x \right\|_{\psi_1}
	\le
\cV_x \cV_y
.
$$
Combining all above inequalities and using Lemma \ref{lemm:triangle_inequality} yields
$$
\left\|\vv^\top \widetilde{\Ab} \vv\right\|_{\psi_1}
	\le
2 \cV_x \cV_y
	,\qquad
\left\|\widetilde{\Ab} \vv\right\|_{\psi_1}
	\le
2 \cV_x \cV_y
	,\qquad
\left\|\vv^\top \widetilde{\Bb}' \vv\right\|_{\psi_1}
	\le
\cV_x^2 + \cV_y^2
	,\qquad
\left\|\widetilde{\Bb}' \vv\right\|_{\psi_1}
	\le
\cV_x^2 + \cV_y^2
.
$$
By applying Lemmas \ref{lemm:Orlicz_norm_property} and \ref{lemm:psialpha}, in CCA problem we have stochastic scaled-gradient satisfying
\begin{align*}
\left\|
(\vv^\top \widetilde{\Bb}' \vv) \widetilde{\Ab} \vv - (\vv^\top \widetilde{\Ab} \vv) \widetilde{\Bb}' \vv
\right\|_{\psi_{1 / 2}}
	&\le
G_{1 / 2} \left(
\left\|\vv^\top \widetilde{\Bb}' \vv\right\|_{\psi_1} 
\left\|\widetilde{\Ab} \vv\right\|_{\psi_1}
+
\left\|\vv^\top \widetilde{\Ab} \vv\right\|_{\psi_1} 
\left\|\widetilde{\Bb}' \vv\right\|_{\psi_1}
\right)
	\\&\le
400 (\cV_x^2 + \cV_y^2) \cV_x \cV_y
.
\end{align*}
Hence Assumption \ref{ass:se} holds for $\cV = 400 (\cV_x^2 + \cV_y^2) \cV_x \cV_y$ and $\alpha = 1 / 2$.
\end{proof}

\pb\section{Summary}\label{sec_summary}
We have presented the Stochastic Scaled-Gradient Descent (SSGD) algorithm for minimizing a constrained nonconvex objective function. 
Comparing with classical stochastic gradient descent, our method only requires access to an unbiased estimate of a scaled gradient, allowing access to a broader range of applications. 
The proposed algorithm requires only a single pass through the data and is memory-efficient, with storage complexity linearly dependent on the ambient dimensionality of the problem. 
For a class of nonconvex stochastic optimization problems, we establish local convergence rates of the proposed algorithm to local minimizers and we prove asymptotic normality of the trajectory average. 
We also investigated the rate of escape of saddle points for SSGD defined on unit sphere.
An application to the generalized eigenvector problem to canonical correlation analysis is investigated both theoretically and numerically.
In near future, we will study the global convergence of SSGD for generic Riemannian manifolds, as well as exploring alternative methods for other applications.

\pb\section*{Acknowledgements}
We thank the Department of Electrical Engineering and Computer Sciences at UC Berkeley for COVID-19 accommodations during which time this work is completed. We thank Tong Zhang, Huizhuo Yuan, Yuren Zhou for inspiring discussions at various stages of this project. This work was supported in part by the Mathematical Data Science program of the Office of Naval Research under grant number N00014-18-1-2764.
%%%%%%%%%%%%%%%%%%%%%%%%%%%%%%%%%%%%%%%%%%%%%%%%%%%%%%%%%%%%%%%%%%%%%

\pb
\bibliographystyle{apalike}
\bibliography{SAILreferences}

\appendix

\pb\section{Preliminaries for Orlicz-$\psi_\alpha$ Norm}\label{sec:orlicz}
Of similar style as \cite[\S E]{li2021stochastic} we collect in this section some facts for Orlicz-$\psi_\alpha$ norm for our usage.
We start with its definition:
 
\begin{definition}[Orlicz $\psi_\alpha$-norm]\label{defi:orlicz}
For a continuous, monotonically increasing and convex function $\psi(x)$ defined for all $x > 0$ satisfying $\psi(0) = 0$ and $\lim_{x \to \infty} \psi(x) = \infty$, we define the Orlicz $\psi$-norm for a random variable $X$ as
$$
\|X\|_\psi
    \equiv
\inf\left\{
    K > 0: \Exs \psi\left( \frac{|X|}{K} \right) \le 1
\right\}
.
$$
As a commonly used special case, we consider function $\psi_\alpha(x) \equiv \exp(x^\alpha) - 1$ and define the Orlicz $\psi_\alpha$-norm for a random variable $X$ as
$$
\|X\|_{\psi_\alpha}
\equiv
\inf\left\{
K > 0:
\Exs \exp\left(
\frac{|X|^\alpha}{K^\alpha}
\right)
\le
2
\right\}
.
$$
\end{definition}

\begin{lemma}\label{lemm:triangle_inequality}
When $\psi(x)$ is monotonically increasing and convex for $x > 0$, for any random variables $X, Y$ with finite Orlicz $\psi$-norm, the triangle inequality holds
$$
\|X + Y\|_\psi
\le
\|X\|_\psi + \|Y\|_\psi
.
$$
For all $\alpha \ge 1$, the above inequality holds when $\|\cdot\|_\psi$ is taken as the Orlicz $\psi_\alpha$-norm.
\end{lemma}

\begin{proof}[Proof of Lemma \ref{lemm:triangle_inequality}]
Let $K_1, K_2$ denote the Orlicz $\psi$-norms of $X$ and $Y$.
Because $\psi(x)$ is monotonically increasing and convex, we have
\begin{align*}
\psi\left(
\frac{|X + Y|}{K_1 + K_2}
\right)
&\le
\psi\left(
\frac{K_1}{K_1 + K_2} \cdot \frac{|X|}{K_1}
+
\frac{K_2}{K_1 + K_2} \cdot \frac{|Y|}{K_2}
\right)\\
&\le
\frac{K_1}{K_1 + K_2}
\cdot
\psi\left(
\frac{|X|}{K_1}
\right)
+
\frac{K_2}{K_1 + K_2}
\cdot
\psi\left(
\frac{|Y|}{K_2}
\right)
,
\end{align*}
which implies
$$
\Exs\psi\left(
\frac{|X + Y|}{K_1 + K_2}
\right)
\le
1
,\qquad\text{i.e. }
\|X + Y\|_\psi
\le
\|X\|_\psi + \|Y\|_\psi
,
$$
yielding the lemma.
\end{proof}

\begin{lemma}\label{lemm:Orlicz_norm_property}
Let $X$ and $Y$ be random variables with finite $\psi_\alpha$-norm for some $\alpha \ge 1$, then
$$
\|XY\|_{\psi_{\alpha / 2}}
\le
\|X\|_{\psi_\alpha} \|Y\|_{\psi_\alpha}
.
$$
\end{lemma}

\begin{proof}[Proof of Lemma \ref{lemm:Orlicz_norm_property}]
Denote $A \equiv X / \|X\|_{\psi_\alpha}$, $B \equiv Y / \|Y\|_{\psi_\alpha}$, then $\|A\|_{\psi_\alpha} = \|B\|_{\psi_\alpha} = 1$.
Using the elementary inequality
$$
|AB|
\le
\frac14 (|A| + |B|)^2
,
$$
and the triangle inequality in Lemma \ref{lemm:triangle_inequality} we have that
$$
\|AB\|_{\psi_{\alpha / 2}}
\le
\frac14 \|(|A| + |B|)^2\|_{\psi_{\alpha / 2}}
=
\frac14 \||A| + |B|\|_{\psi_\alpha}^2
\le
\frac14 (\|A\|_{\psi_\alpha} + \|B\|_{\psi_\alpha})^2
=
1
.
$$
Multiplying both sides of the inequality by $\|X\|_{\psi_\alpha} \|Y\|_{\psi_\alpha}$ gives the desired result.
\end{proof}

\begin{lemma}\label{lemm:psialpha}
For any random variables $X, Y$ with finite Orlicz $\psi_\alpha$-norm, the following inequalities hold
$$
\|X + Y\|_{\psi_\alpha}
\le
\log_2^{1 / \alpha}(1 + e^{1 / \alpha}) (\|X\|_{\psi_\alpha} + \|Y\|_{\psi_\alpha})
,\qquad
\|\Exs X\|_{\psi_\alpha}
\le
\log_2^{1 / \alpha}(1 + e^{1 / \alpha}) \|X\|_{\psi_\alpha}
,
$$
and
$$
\|X - \Exs X\|_{\psi_\alpha}
\le
\log_2^{1 / \alpha}(1 + e^{1 / \alpha}) \left(
1 + \log_2^{1 / \alpha}(1 + e^{1 / \alpha})
\right) \|X\|_{\psi_\alpha}
.
$$
\end{lemma}

\begin{proof}[Proof of Lemma \ref{lemm:psialpha}]
Recall that when $\alpha \in (0,1)$, $\psi_\alpha(x)$ does \textit{not} satisfy convexity when $x$ is around 0.
Let $\widetilde \psi_\alpha(x)$ be
$$
\widetilde \psi_\alpha(x)
=
\left\{
\begin{array}{ll}
\exp(x^\alpha) - 1
&
x \ge x_*
\\
\frac{x}{x_*}
\left(
 \exp(x_*^\alpha) - 1
\right)
&
x \in [0, x_*)
\end{array}
\right.
.
$$
for some appropriate $x_* > 0$, so as to make the function convex.
Here $x_*$ is chosen such that the tangent line of function $\psi_\alpha$ at $x_*$ passes through origin, i.e.
$$
\psi_\alpha'(x_*)
=
\alpha x_*^{\alpha - 1} \exp(x_*^\alpha)
=
\frac{\exp(x_*^\alpha) - 1}{x_*}
=
\widetilde \psi_\alpha'(x_*)
.
$$
Simplifying it gives us a transcendental equation
$$
(1 - \alpha x_*^\alpha) \exp(x_*^\alpha)
=
1
.
$$
We easily find that $x_*^\alpha \le 1 / \alpha$.
Because $\psi_\alpha(x)$ is concave on $\left(0, (\frac{1}{\alpha} - 1)^{1 / \alpha}\right)$ and convex on $((\frac{1}{\alpha} - 1)^{1 / \alpha}, \infty)$, we have $\psi_\alpha(x) \ge \widetilde \psi_\alpha(x) \ge 0$ for all $x \ge 0$, and hence
\beq\label{eq:psi_tildepsi}
0
\le
\psi_\alpha(x) - \widetilde \psi_\alpha(x)
\le
\psi_\alpha(x_*)
\le
e^{1 / \alpha} - 1
.
\eeq
Let $K_1, K_2$ denote the Orlicz $\psi_\alpha$-norms of $X$ and $Y$, then
$$
\Exs\widetilde \psi_\alpha \left( \frac{|X|}{K_1} \right)
	\le
\Exs\psi_\alpha \left( \frac{|X|}{K_1} \right)
	\le
1
	,\qquad
\Exs\widetilde \psi_\alpha \left( \frac{|Y|}{K_2} \right)
	\le
\Exs\psi_\alpha \left( \frac{|Y|}{K_2} \right)
	\le
1
.
$$
By applying the triangle inequality in Lemma \ref{lemm:triangle_inequality} and using \eqref{eq:psi_tildepsi}, we have
\begin{gather*}
\Exs\psi_\alpha\left(\frac{|X + Y|}{K_1 + K_2}\right)
	\le
\Exs\widetilde \psi_\alpha\left(\frac{|X + Y|}{K_1 + K_2}\right)
+
e^{1 / \alpha} - 1
	\le
e^{1 / \alpha}
	,\\
\Exs\psi_\alpha\left(\frac{|\Exs X|}{K_1}\right)
	\le
\Exs\widetilde \psi_\alpha\left(\frac{|\Exs X|}{K_1}\right)
	+
e^{1 / \alpha} - 1
	\le
e^{1 / \alpha}
.
\end{gather*}
By applying Jensen's inequality to concave function $J_\alpha(z) = z^{\log_{1 + e^{1 / \alpha}}2}$, we have
\begin{align*}
\Exs \psi_\alpha \left(
	\frac{|X + Y|}{\log_2^{1 / \alpha}(1 + e^{1 / \alpha}) (K_1 + K_2)}
\right)
	&=
\Exs J_\alpha\left( \exp\left(
	\frac{|X + Y|^\alpha}{(K_1 + K_2)^\alpha}
\right) \right) - 1
	\\&\le
J_\alpha \left( \Exs\exp\left(\frac{|X + Y|^\alpha}{(K_1 + K_2)^\alpha}\right) \right) - 1
	\le
1
,
\end{align*}
and
$$
\Exs \psi_\alpha \left(\frac{|\Exs X|}{\log_2^{1 / \alpha}(1 + e^{1 / \alpha}) K_1}\right)
	=
\Exs J_\alpha\left( \exp\left(\frac{|\Exs X|^\alpha}{K_1^\alpha}\right) \right) - 1
	\le
J_\alpha \left( \Exs\exp\left(\frac{|\Exs X|^\alpha}{K_1^\alpha}\right) \right) - 1
	\le
1
,
$$
which implies
$$
\|X + Y\|_{\psi_\alpha}
	\le
\log_2^{1 / \alpha}(1 + e^{1 / \alpha}) (\|X\|_{\psi_\alpha} + \|Y\|_{\psi_\alpha})
	,\qquad
\|\Exs X\|_{\psi_\alpha}
	\le
\log_2^{1 / \alpha}(1 + e^{1 / \alpha}) \|X\|_{\psi_\alpha}
,
$$
and
$$
\|X - \Exs X\|_{\psi_\alpha}
	\le
\log_2^{1 / \alpha}(1 + e^{1 / \alpha}) (\|X\|_{\psi_\alpha} + \|\Exs X\|_{\psi_\alpha})
	\le
\log_2^{1 / \alpha}(1 + e^{1 / \alpha}) \left(
	1 + \log_2^{1 / \alpha}(1 + e^{1 / \alpha})
\right) \|X\|_{\psi_\alpha}
.
$$
\end{proof}

Now we proceed with the definition of Orlicz $\psi_\alpha$-norm for random vectors.

\begin{definition}
For a random vector $\Xb \in \real^d$, its Orlicz $\psi_\alpha$-norm is defined as
$$
\|\Xb\|_{\psi_\alpha}
\equiv
\inf\left\{
K > 0:
\Exs \exp\left(
\frac{\|\Xb\|^\alpha}{K^\alpha}
\right)
\le
2
\right\}
.
$$
\end{definition}
Seeing the above definition, a random vector $\Xb$ is called \emph{sub-Gaussian} if $\|\Xb\|_{\psi_2} < \infty$, and is called \emph{sub-Exponential} if $\|\Xb\|_{\psi_1} < \infty$.

\begin{remark}
We notice that $\|\Xb\|_{\psi_\alpha}$ equals to the Orlicz $\psi_\alpha$-norm of random variable (scalar) $\|\Xb\|$.
Using this relation, we can easily extend all above results of random variables to random vectors with the same positive factors and dependency on $\alpha$.
\end{remark}

\pb\section{Estimation of the Strict-saddle Parameters}\label{sec_proof,lemm_littlegrad}

The goal of this section is to detail the proof of Lemma \ref{lemm_littlegrad} that estimates the strict-saddle parameters.
We first compute the manifold gradient and Hessian in the following Lemma \ref{lemm:GEVcharacterize}:

\begin{lemma}\label{lemm:GEVcharacterize}
The manifold gradient and Hessian can be computed as
\beq\label{GEVgrad}
g(\xb)
=
-2\frac{(\xb^\top \Bb \xb) \Ab - (\xb^\top \Ab \xb) \Bb }{(\xb^\top \Bb \xb)^2} \xb
,
\eeq
\beq\label{GEVhessian}
\begin{split}
\cH(\xb)
&=
-2\frac{(\xb^\top \Bb \xb) \Ab - (\xb^\top \Ab \xb) \Bb
+2( \Ab \xb \xb^\top \Bb - \Bb \xb \xb^\top \Ab )}{(\xb^\top \Bb \xb)^2}
+8\frac{\left[(\xb^\top \Bb \xb) \Ab - (\xb^\top \Ab \xb) \Bb
\right] \xb \xb^\top \Bb
}{(\xb^\top \Bb \xb)^3}
.
\end{split}
\eeq

\end{lemma}

\begin{proof}
The constrained optimization problem has $c(\xb) = \|\xb\|^2 - 1$ so the Lagrangian is
$$
\cL(\xb;\mu)
=
-\frac{\xb^\top \Ab \xb}{\xb^\top \Bb \xb} - \mu (\xb^\top \xb - 1)
.
$$
According to the constrained optimization theory in \citet{NOCEDAL-WRIGHT}, since (i) there is one constraint (ii) the gradient $g(\xb) = 2\xb$ on constraint has constant norm 2, it satisfies some 2-RLICQ condition.
The feasible value of Lagrangian multiplier $\mu^*(\xb)$ has
$$
\mu^*(\xb)
=
\arg\min_\mu \|\nabla_\xb \cL(\xb,\mu)\|^2.
$$
Let
$$
\Lambda(\xb)
=
\frac{(\xb^\top \Bb \xb) \Ab - (\xb^\top \Ab \xb) \Bb }{(\xb^\top \Bb \xb)^2}.
$$
Then we have
$$
\nabla \cL(\xb;\mu)
=
-2 \Lambda(\xb) \xb - 2\mu \xb
,
$$
and hence
\begin{align*}
\left\|\nabla_\xb \cL(\xb;\mu)\right\|^2
=
4 \left\| \Lambda(\xb) \xb + \mu \xb \right\|^2
&=
4 \| \Lambda(\xb) \xb + \mu \xb\|^2
\\&=
4 \left(
\xb^\top \Lambda(\xb) \Lambda(\xb) \xb
+ 2 (\xb^\top \Lambda(\xb) \xb ) \mu
+ (\xb^\top \xb) \mu^2
\right)
.
\end{align*}
Solving this problem gives $\mu^*(\xb)$ for $\xb\in\cS^{d-1}$:
$$
\mu^*(\xb)
=
- \xb^\top \Lambda(\xb) \xb
=
-
\frac{(\xb^\top \Bb \xb) \xb^\top \Ab \xb - (\xb^\top \Ab \xb) \xb^\top \Bb \xb}{(\xb^\top \Bb \xb)^2}
=
0
.
$$
The manifold gradient can hence be computed as
$$
g(\xb)
=
\nabla L(\xb;\mu) \big|_{\mu = \mu^*(\xb)}
=
-2\Lambda(\xb) \xb - 2 \mu^*(\xb) \xb
=
-2\frac{(\xb^\top \Bb \xb) \Ab - (\xb^\top \Ab \xb) \Bb }{(\xb^\top \Bb \xb)^2} \xb
,
$$
concluding \eqref{GEVgrad}.
For manifold Hessian, we can compute it as
\begin{align*}
\cH(\xb)
&=
\nabla^2 L(\xb;\mu) \big|_{\mu = \mu^*(\xb)}
=
-2 \nabla \left[
\frac{(\xb^\top \Bb \xb) \Ab - (\xb^\top \Ab \xb) \Bb }{(\xb^\top \Bb \xb)^2} \xb
\right]
\\&=
-2
\frac{(\xb^\top \Bb \xb) \Ab - (\xb^\top \Ab \xb) \Bb
+
2( \Ab \xb \xb^\top \Bb - \Bb \xb \xb^\top \Ab )}{(\xb^\top \Bb \xb)^2}
+
4\frac{\left[(\xb^\top \Bb \xb) \Ab - (\xb^\top \Ab \xb) \Bb
\right] \xb \xb^\top (2 \Bb )
}{(\xb^\top \Bb \xb)^3}
.
\end{align*}
This proves \eqref{GEVhessian} and concludes the lemma.

\end{proof}

We prove the Hessian smoothness and give the Lipschitz constant for both manifold gradient and Hessian, as in the following lemmas.

\begin{lemma}\label{lemm:HessianLip}
There are Lipschitz constants
$$
L_G
\equiv
\frac{28\|\Ab\|\|\Bb\|^2}{ \lambda_{\min}^3(\Bb)}
,\qquad
L_H
\equiv
\frac{56\|\Ab\|\|\Bb\|^3}{ \lambda_{\min}^4(\Bb)}
,
$$
such that for all $\zb, \zb_1, \zb_2 \in \cS^{d-1}$ we have
\beq\label{gradLip}
\|\cH(\zb)\| \le L_G
,
\eeq
and
\beq\label{hessLip}
\|\cH(\zb_1) - \cH(\zb_2)\|
\le
L_H \| \zb_1 - \zb_2\|
.
\eeq
In addition, we have from above two
\beq\label{gradhessLip}
\left\|
P_{\cT(\zb)}^\top \cH(\zb) P_{\cT(\zb)}
-
P_{\cT(\zb')}^\top \cH(\zb') P_{\cT(\zb')}
\right\|
\le
(2L_G + L_H) \|\zb - \zb'\|
.
\eeq
In fact, in this lemma one can replace $\|\Ab\|$ by the norm $\|\Ab - c \Bb\|$ for any constant scalar $c$.

\end{lemma}

\begin{proof}[Proof of Lemma \ref{lemm:HessianLip}]
Note
$$
\|g(\xb)\|
\le
\frac{\|\Bb\|}{\lambda_{\min}^2(\Bb)} \|\Ab\|
,
$$
and
$$
\|\cH(\xb)\|
\le
2\frac{2\|\Bb\|\|\Ab\| + 4\|\Ab\| \|\Bb\| }{\lambda_{\min}^2(\Bb)}
+
8\frac{ 2\|\Bb\|\|\Ab\| \|\Bb\|}{\lambda_{\min}^3(\Bb)}
\le
\frac{28\|\Bb\|^2 }{\lambda_{\min}^3(\Bb)}\|\Ab\|
,
$$
so we conclude \eqref{gradLip} from mean-value theorem.

Moreover, for an arbitrary unit vector $\vv$,
\begin{align*}
&\quad
\left\| \cH(\xb) \vv - \cH(y) \vv \right\|
\\&\le
2\left\|
\frac{(\xb^\top \Bb \xb) \Ab - (\xb^\top \Ab \xb) \Bb
+2( \Ab \xb \xb^\top \Bb - \Bb \xb \xb^\top \Ab )}{(\xb^\top \Bb \xb)^2} \vv
-
\frac{(y^\top \Bb y) \Ab - (y^\top \Ab y) \Bb
+2( \Ab y y^\top \Bb - \Bb y y^\top \Ab )}{(y^\top \Bb y)^2} \vv
\right\|
\\&\quad\quad
+8\left\|
\frac{\left[(\xb^\top \Bb \xb) \Ab - (\xb^\top \Ab \xb) \Bb
\right] \xb \xb^\top \Bb
}{(\xb^\top \Bb \xb)^3}\vv
-
\frac{\left[(y^\top \Bb y) \Ab - (y^\top \Ab y) \Bb
\right] y y^\top \Bb
}{(y^\top \Bb y)^3}\vv
\right\|
\equiv
\mbox{I} + \mbox{II}
.
\end{align*}
Note
\begin{align*}
&\quad
\nabla\left[
\frac{(\xb^\top \Bb \xb) \Ab - (\xb^\top \Ab \xb) \Bb
+2( \Ab \xb \xb^\top \Bb - \Bb \xb \xb^\top \Ab )}{(\xb^\top \Bb \xb)^2} \vv
\right]
\\&=
\frac{2(\xb^\top \Bb ) \Ab \vv - 2(\xb^\top \Ab ) \Bb \vv
+2( (\xb^\top \Bb \xb\vv) \Ab - (\xb^\top \Ab \bv) \Bb )
+2( \Ab \xb \vv^\top \Bb -  \Bb \xb \vv^\top \Ab)
}{(\xb^\top \Bb \xb)^2}
\\&\quad
-2\frac{\left[(\xb^\top \Bb \xb) \Ab - (\xb^\top \Ab \xb) \Bb
+2( \Ab \xb \xb^\top \Bb - \Bb \xb \xb^\top \Ab ) \right] \xb\vv^\top (2\Bb)}{(\xb^\top \Bb \xb)^3}
,
\end{align*}
whose norm is bounded by $36\|\Ab\|\|\Bb\|^2 / \lambda_{\min}^3(\Bb)$, and
\begin{align*}
	&\quad
\nabla\left[
\frac{\left[(\xb^\top \Bb \xb) \Ab - (\xb^\top \Ab \xb) \Bb
\right] \xb \xb^\top \Bb
}{(\xb^\top \Bb \xb)^3}\vv
\right]
	=
\nabla\left[
\frac{(\xb^\top \Bb \bv) \left[(\xb^\top \Bb \xb) \Ab \xb - (\xb^\top \Ab \xb) \Bb \xb
\right]
}{(\xb^\top \Bb \xb)^3}
\right]
	\\&=
\frac{ \left[(\xb^\top \Bb \xb) \Ab - (\xb^\top \Ab \xb) \Bb
\right] \xb \vv^\top \Bb
}{(\xb^\top \Bb \xb)^3}
+
\frac{(\xb^\top \Bb \bv) \left[
(\xb^\top \Bb \xb) \Ab - (\xb^\top \Ab \xb) \Bb
+
2( \Ab \xb \xb^\top \Bb - \Bb \xb \xb^\top \Ab )
\right]
}{(\xb^\top \Bb \xb)^3}
	\\&\quad
-3
\left[
\frac{ (\xb^\top \Bb \bv)\left[(\xb^\top \Bb \xb) \Ab - (\xb^\top \Ab \xb) \Bb
\right] \xb
\xb^\top (2\Bb)}{(\xb^\top \Bb \xb)^4}
\right]
,
\end{align*}
whose norm is thus bounded by $20\|\Ab\|\|\Bb\|^3 / \lambda_{\min}^4(\Bb)$.
Again by mean value theorem we have
$$
\|\mbox{I}\|
\le
\frac{36\|\Ab\|\|\Bb\|^2}{ \lambda_{\min}^3(\Bb)}
\|\xb - \yb\|
,
$$
and
$$
\|\mbox{II}\|
\le
\frac{20\|\Ab\|\|\Bb\|^3}{ \lambda_{\min}^4(\Bb)}
\|\xb - \yb\|
,
$$
so
$$
\left\| \cH(\xb) \vv - \cH(\yb) \vv \right\|
\le
\frac{56\|\Ab\|\|\Bb\|^3}{ \lambda_{\min}^4(\Bb)}
\|\xb - \yb\|
,
$$
which concludes \eqref{hessLip} via the definition of operator norm.

Lastly to conclude \eqref{gradhessLip}, we utilize the properties of projection matrices, $\|P_{\cT(\zb)}\| \le 1$, $\|P_{\cT(\zb_1)} - P_{\cT(\zb_2)}\| \le \|\zb_1 - \zb_2\|$ and hence from matrix operator theory
\begin{align*}
&\quad
\left\|
P_{\cT(\zb)}^\top \cH(\zb) P_{\cT(\zb)}
-
P_{\cT(\zb')}^\top \cH(\zb') P_{\cT(\zb')}
\right\|
\\&\le
\left\|
P_{\cT(\zb)}^\top \cH(\zb) P_{\cT(\zb)}
-
P_{\cT(\zb)}^\top \cH(\zb) P_{\cT(\zb')}
\right\|
\\&\quad
+
\left\|
P_{\cT(\zb)}^\top \cH(\zb) P_{\cT(\zb')}
-
P_{\cT(\zb)}^\top \cH(\zb') P_{\cT(\zb')}
\right\|
+
\left\|
P_{\cT(\zb)}^\top \cH(\zb') P_{\cT(\zb')}
-
P_{\cT(\zb')}^\top \cH(\zb') P_{\cT(\zb')}
\right\|
%%%%%%%%
\\&\le
\|P_{\cT(\zb)}^\top\| \|\cH(\zb)\| \|P_{\cT(\zb)} - P_{\cT(\zb')}\|
\\&\quad+
\|P_{\cT(\zb)}^\top \| \left\|\cH(\zb) - \cH(\zb')\right\| \|P_{\cT(\zb')}\|
+
\left\|
(P_{\cT(\zb)} - P_{\cT(\zb')})^\top
\right\|
\|\cH(\zb')\| \|P_{\cT(\zb')}\|
%%%%%%%%
\\&\le
L_G \|\zb - \zb'\|
+
L_H \|\zb - \zb'\|
+
L_G \|\zb - \zb'\|
\\&=
(2L_G + L_H) \|\zb - \zb'\|
.
\end{align*}
We complete our proof.

\end{proof}

We now come to explore what the small gradient condition $\|g(\xb)\| \le \gammai$, where $g(\cdot)$ is defined in \eqref{GEVgrad}, means for a point $\xb$ in the GEV Problem.
We first analyze the case where $\Bb$ is the identity matrix, which reduces to the classical Eigenvector Problem.
Define for convenience
\beq\label{ep1}
\gamma_1
\equiv
\left(
\frac{ \|\Bb\| }{\lambda_{\min}(\Bb) }
\right)^{1/2}
\frac{\gammai}{2\lambda_{\text{gap}}}
.
\eeq

\begin{lemma}\label{lemm:smallgradEV}
When $\Bb = \Ib$, we have under $\|\ww\|=1$, an arbitrary constant $\gamma_1 \in (0, 1/2)$ and
$$
\left\|
\Lambda \ww - \frac{\ww^\top \Lambda \ww}{\ww^\top \ww} \ww
\right\|
\le
\lambda_{\text{gap}}\gamma_1
,
$$
and for some $j=1,\dots,d$ (for consistency we define $\lambda_0 = \lambda_1$ and $\lambda_{d+1} = \lambda_d$)
$$
\frac{\ww^\top \Lambda \ww}{\ww^\top \ww} \in \left[
\frac{\lambda_{j-1} + \lambda_j}{2}
,
\frac{\lambda_j + \lambda_{j+1}}{2}
\right]
,
$$
together imply
$$
(\eb_j^\top \ww)^2 \ge 1 - 4 \gamma_1^2
.
$$
\end{lemma}

\begin{proof}
Denote till the rest of this proof $w_i = \eb_i^\top \ww$.
Note we have by
\begin{align*}
\lambda_{\text{gap}}^2 \gamma_1^2
&\ge
\left\|
\Lambda \ww - \frac{\ww^\top \Lambda \ww}{\ww^\top \ww} \ww
\right\|^2
=
\sum_{i=1}^d \left( \lambda_i - \frac{\ww^\top \Lambda \ww}{\ww^\top \ww} \right)^2 w_i^2
\\&\ge
\sum_{i=1}^{j-1} \left( \lambda_i - \frac{\ww^\top \Lambda \ww}{\ww^\top \ww} \right)^2 w_i^2
+
\sum_{i=j+1}^d \left( \lambda_i - \frac{\ww^\top \Lambda \ww}{\ww^\top \ww} \right)^2 w_i^2
\\&\ge
\sum_{i=1}^{j-1} \left( \lambda_i - \frac{\lambda_{j-1} + \lambda_j}{2} \right)^2 w_i^2
+
\sum_{i=j+1}^d \left( \lambda_i - \frac{\lambda_j + \lambda_{j+1}}{2} \right)^2 w_i^2
\\&\ge
\left( \frac{ \lambda_j - \lambda_{j-1}}{2} \right)^2 \sum_{i=1}^{j-1} w_i^2
+
\left( \frac{ \lambda_{j+1} - \lambda_j}{2} \right)^2 \sum_{i=j+1}^d w_i^2
\ge
\frac{ \lambda_{\text{gap}}^2}{4} \left(1 - w_j^2\right)
.
\end{align*}
This implies the lemma immediately.
\end{proof}

To study the case of general $\Bb$, we first introduce an auxiliary lemma.

\begin{lemma}\label{lemm:equivnorm}
Given two norms $\|\cdot\|_1, \|\cdot\|_2$ that are equivalent: there are constants $C_L, C_U > 0$ such that for every nonzero vector $\vv$, $C_L\|\vv\|_2 \le \|\vv\|_1 \le C_U\|\vv\|_2$.
Then for two given nonzero vectors $\ww_1, \ww_2$, we have
$$
\left\|
\frac{\ww_1}{\|\ww_1\|_1}
-
\frac{\ww_2}{\|\ww_2\|_1}
\right\|_1
\le
2C_L^{-1} C_U \left\|
\frac{\ww_1}{\|\ww_1\|_2}
-
\frac{\ww_2}{\|\ww_2\|_2}
\right\|_2
$$
\end{lemma}

\begin{proof}
Without loss of generality we set $\|\ww_1\|_2 = 1 = \|\ww_2\|_2$.
Then using triangle inequality we have
\begin{align*}
LHS
&=
\left\|
\frac{\ww_1}{\|\ww_1\|_1}
-
\frac{\ww_2}{\|\ww_2\|_1}
\right\|_1
=
\frac{\left\|
\|\ww_2\|_1\ww_1
-
\|\ww_1\|_1\ww_2
\right\|_1
}
{\|\ww_1\|_1\|\ww_2\|_1}
\\&=
\frac{\left\|
\|\ww_2\|_1\ww_1
-
\|\ww_1\|_1\ww_1
+
\|\ww_1\|_1\ww_1
-
\|\ww_1\|_1\ww_2
\right\|_1
}
{\|\ww_1\|_1\|\ww_2\|_1}
\\&\le
\frac{
\left|\|\ww_2\|_1 - \|\ww_1\|_1\right| \|\ww_1\|_1
+
\|\ww_1\|_1 \left\|\ww_1 - \ww_2 \right\|_1
}
{\|\ww_1\|_1\|\ww_2\|_1}
\\&\le
\frac{
2\|\ww_1 - \ww_2\|_1 \|\ww_1\|_1
}
{\|\ww_1\|_1\|\ww_2\|_1}
=
2\|\ww_2\|_1^{-1} \cdot \|\ww_1 - \ww_2\|_1
\le
2C_L^{-1} \|\ww_2\|_2^{-1} \cdot C_U \|\ww_1 - \ww_2\|_2
=
RHS
.
\end{align*}
\end{proof}

We conclude the following lemma.

\begin{lemma}\label{lemm:smallgradGEV}
We have for $\xb\in \cS^{d-1}$,
\beq\label{epcond}
\gammai \in \left(
0,
\left(
\frac{ \|\Bb\| }{\lambda_{\min}(\Bb) }
\right)^{-1/2} \lambda_{\text{gap}}
\right)
,
\eeq
and $\|g(\xb)\| \le \gammai$ implies that there exists at least one $j = 1,\dots, d$ such that
\beq\label{innerB}
(\eb_j^\top \Bb^{1/2} \xb)^2 \ge (1 - 4 \gamma_1^2) \|\Bb^{1/2} \xb\|^2
.
\eeq
Furthermore, we have that there exists at least one $j = 1,\dots, d$ such that
\beq\label{xvjclose}
\min\left( \| \xb - \vv_j \|, \| \xb + \vv_j \| \right)
\le
4\sqrt{2}\left(
\frac{\|\Bb\| }{ \lambda_{\min}(\Bb) }
\right)^{1/2}
\cdot
\gamma_1
.
\eeq
\end{lemma}

\begin{proof}
Since $\Bb$ is positive definite, letting in \eqref{lemm:smallgradEV} $w = \Bb^{1/2} \xb / \|\Bb^{1/2} \xb\|$, we have $\|\ww\|=1$ and recall that $\Ab = \Bb^{1/2}\Lambda \Bb^{1/2}$
\begin{align*}
\left\|
\Lambda \ww - \frac{\ww^\top \Lambda \ww}{\ww^\top \ww} \ww
\right\|
&=
\|\Bb^{1/2} \xb\|
\left\|
\Bb^{-1/2}\left(
\frac{\Bb^{1/2}\Lambda \Bb^{1/2}}{\xb^\top \Bb \xb} \xb
- \frac{\xb^\top \Bb^{1/2}\Lambda \Bb^{1/2} \xb}{(\xb^\top \Bb \xb)^2} \Bb \xb
\right)
\right\|
\\&\le
\left(
\frac{ \|\Bb\| }{\lambda_{\min}(\Bb)}
\right)^{1/2}
\left\|
\frac{(\xb^\top \Bb \xb) \Ab - (\xb^\top \Ab \xb) \Bb }{(\xb^\top \Bb \xb)^2} \xb
\right\|
\le
\left(
\frac{ \|\Bb\| }{\lambda_{\min}(\Bb)}
\right)^{1/2}
\frac{\gammai}{2}
=
\lambda_{\text{gap}} \gamma_1
.
\end{align*}
Note \eqref{epcond} gives $\gamma_1 \in (0, 1/2)$, and hence applying Lemma \ref{lemm:smallgradEV} gives the following: there is at least one $j=1,\dots,d$ such that
$(\eb_j^\top \ww)^2 \ge 1 - 4 \gamma_1^2$.
Translating this back in terms of $\xb$ concludes \eqref{innerB}.

To conclude \eqref{xvjclose} we note \eqref{innerB} gives if $\left< \zb_1, \zb_2 \right>_B \equiv \zb_1^\top \Bb \zb_2$ and $\|\zb\|_\Bb \equiv \left< z, z\right>_B^{1/2}$:
$$
\left< \eb_j^\top \Bb^{-1/2}, \frac{\xb}{\|\xb\|_\Bb} \right>_B^2
\ge
1 - 4\gamma_1^2
,
$$
so
\begin{align*}
\left\| \eb_j^\top \Bb^{-1/2} \pm \frac{\xb}{\|\xb\|_\Bb} \right\|_\Bb^2
&=
\left\| \eb_j^\top \Bb^{-1/2} \right\|_\Bb^2 + \left\| \frac{\xb}{\|\xb\|_\Bb} \right\|_\Bb^2
\pm
2 \left< \eb_j^\top \Bb^{-1/2}, \frac{\xb}{\|\xb\|_\Bb} \right>_\Bb
\\&=
2 \pm
2 \left< \eb_j^\top \Bb^{-1/2}, \frac{\xb}{\|\xb\|_\Bb} \right>_\Bb
,
\end{align*}
and hence using $1 - \sqrt{1-t} \le t$ for $t\in [0,1]$
$$
\min \left\| \eb_j^\top \Bb^{-1/2} \pm \frac{\xb}{\|\xb\|_\Bb} \right\|_\Bb^2
=
2 - 2 \left|\left< \eb_j^\top \Bb^{-1/2}, \frac{\xb}{\|\xb\|_\Bb} \right>_\Bb\right|
=
2 - 2\sqrt{1-4\gamma_1^2}
\le
8\gamma_1^2
.
$$
Using this and applying Lemma \ref{lemm:equivnorm} with $\|\cdot\|_1 = \|\cdot\|$ and $\|\cdot\|_2 = \|\cdot\|_\Bb$ we have $\lambda_{\max}^{-1/2}(\Bb) \|\vv\|_\Bb \le \|\vv\| \le \lambda_{\min}^{-1/2}(\Bb) \|\vv\|_\Bb$ and hence for two given nonzero vectors (in the Euclidean norm) $\xb$ and $\mp \vv_j = \mp \| \eb_j^\top \Bb^{-1/2} \|^{-1} \eb_j^\top \Bb^{-1/2}$
$$
\min
\|\xb \pm \vv_j \|
\le
2 \left(
\frac{\|\Bb\| }{ \lambda_{\min}(\Bb) }
\right)^{1/2}
\cdot
\min
\left\|
\eb_j^\top \Bb^{-1/2} \pm \frac{\xb}{\|\xb\|_\Bb}
\right\|_\Bb
\le
4\sqrt{2}\left(
\frac{\|\Bb\| }{ \lambda_{\min}(\Bb) }
\right)^{1/2}
\cdot
\gamma_1
.
$$

\end{proof}

Now we finish the proof of Lemma \ref{lemm_littlegrad}.
\begin{proof}[Proof of Lemma \ref{lemm_littlegrad}]
\begin{enumerate}[label=(\roman*)]
\item
We have $ \Ab \xb = (\xb^\top \Ab \xb / \xb^\top \Bb \xb) \Bb \xb$ if and only if $g (\xb) = 0$.
For $\Lambda = \Bb ^{-1/2} \Ab  \Bb ^{-1/2}$ being WLOG diagonal, one can see that for $j=2,\dots,d$ and $\vv_j$ on the unit sphere with $\Ab\vv_j = \lambda_j \Bb\vv_j$,
$$
\cH(\vv_j)
=
-2 \cdot\frac{(\vv_j^\top \Bb \vv_j) \Ab - (\vv_j^\top \Ab \vv_j) \Bb }{(\vv_j^\top \Bb \vv_j)^2}
=
-2 \cdot\frac{ \Ab - \lambda_j \Bb }{\vv_j^\top \Bb \vv_j}
.
$$
Thus
\begin{align*}
&\quad
(\vv_1 - c \vv_j)^\top \cH(\vv_j) (\vv_1 - c \vv_j)
\\&=
-2 \cdot\frac{ (\vv_1 - c \vv_j)^\top (\Ab - \lambda_j \Bb) (\vv_1 - c \vv_j) }{\vv_j ^\top \Bb \vv_j}
\\&=
-2 \cdot\frac{ \vv_1^\top (\Ab - \lambda_j \Bb) \vv_1 }{\vv_j ^\top \Bb \vv_j}
+4c \cdot\frac{ \vv_j^\top (\Ab - \lambda_j \Bb) \vv_1 }{\vv_j ^\top \Bb \vv_j}
-2c^2 \cdot\frac{ \vv_j^\top (\Ab - \lambda_j \Bb) \vv_j }{\vv_j ^\top \Bb \vv_j}
\\&=
-2 \cdot\frac{ \vv_1^\top (\Ab - \lambda_j \Bb) \vv_1 }{\vv_j ^\top \Bb \vv_j}
\\&=
-2(\lambda_1 - \lambda_j) \cdot\frac{  \vv_1^\top \Bb \vv_1 }{\vv_j ^\top \Bb \vv_j}
\le
-2(\lambda_1 - \lambda_2) \cdot\frac{  \lambda_{\min} (\Bb) }{ \|\Bb\| }
.
\end{align*}

In the display above, we use the fact that $\vv_j^\top (\Ab - \lambda_j \Bb) \vv_1 = (\lambda_1 - \lambda_j ) \vv_j^\top \Bb \vv_1 = 0$ and $\vv_j^\top (\Ab - \lambda_j \Bb) \vv_j = (\lambda_j - \lambda_j ) \vv_j^\top \Bb \vv_j = 0$.
By picking $c = \vv_j^\top \vv_1$ such that $P_{\cT(\vv_j)} \vv_1 = \vv_1 - (\vv_j^\top \vv_1) \vv_j = \vv_1 - c \vv_j$, we conclude
$$
\|P_{\cT(\vv_j)} \vv_1\|
=
\sqrt{ 1 + (\vv_j^\top \vv_1)^2 - 2(\vv_j^\top \vv_1)^2 }
=
\sqrt{1 - (\vv_j^\top \vv_1)^2}
\in (0,1]
,
$$
(since $\vv_1 \ne \pm \vv_j$ otherwise $0 = \vv_j^\top \Bb \vv_1 = \pm \vv_1^\top \Bb \vv_1$ which leads to $\vv_1 = 0$ due to the positive definiteness of $\Bb$.)
and hence from the above two displays
\begin{align*}
\vv_1^\top \left[P_{\cT(\vv_j)}^\top \cH(\vv_j) P_{\cT(\vv_j)}\right] \vv_1
&\le
- 2(\lambda_1 - \lambda_2) \cdot\frac{  \lambda_{\min} (\Bb) }{ \|\Bb\| } \|P_{\cT(\vv_j)}(\vv_1)\|^2
.
\end{align*}

\item
To conclude points that are close to $P_{\cT(\vv_j)} \vv_1$, Lemma \ref{lemm:smallgradGEV} gives for $\xb\in \cS^{d-1}$,
\beq\label{epcond}
\gammai \in \left(
0,
\left(
\frac{ \|\Bb\| }{\lambda_{\min}(\Bb) }
\right)^{-1/2} \lambda_{\text{gap}}
\right)
,
\eeq
and $\|g(\xb)\| \le \gammai$ implies that there exists at least one $j = 1,\dots, d$ such that
\beq\label{xvjclose}
\min \| \xb \pm \vv_j \|
\le
4\sqrt{2}\left(
\frac{\|\Bb\| }{ \lambda_{\min}(\Bb) }
\right)^{1/2}
\cdot
\gamma_1
.
\eeq
Without loss of generality we suppose the minus sign in the above display is taken, so $\min \| \xb - \vv_j \|
\le
4\sqrt{2}\left(
\frac{\|\Bb\| }{ \lambda_{\min}(\Bb) }
\right)^{1/2}
\cdot
\gamma_1$.
Then given the definition of $\gamma_1$ in \eqref{ep1} we have from Lemma \ref{lemm:HessianLip} that
\begin{align*}
&\quad
\vv_1^\top
\left[ P_{\cT(\xb)}^\top \cH(\xb) P_{\cT(\xb)} \right]
\vv_1
\\&\le
\vv_1^\top
\left[ P_{\cT(\vv_j)}^\top \cH(\vv_j) P_{\cT(\vv_j)} \right]
\vv_1
+
\left\|
P_{\cT(\vv_j)}^\top \cH(\vv_j) P_{\cT(\vv_j)}
-
P_{\cT(\xb)}^\top \cH(\xb) P_{\cT(\xb)}
\right\|
\\&\le
-2(\lambda_1 - \lambda_2) \cdot\frac{  \lambda_{\min} (\Bb) }{ \|\Bb\| }
+
(2L_G + L_H) \|\xb - \vv_j\|
\\&\le
- (\lambda_1 - \lambda_2) \cdot\frac{  \lambda_{\min} (\Bb) }{ \|\Bb\| }
\le
- \left( \lambda_1 - \lambda_2 \right) \cdot\frac{  \lambda_{\min} (\Bb) }{ \|\Bb\| } \| P_{\cT(\xb)} \vv_1 \|^2
,
\end{align*}
as long as (combined with \eqref{xvjclose})
$$
4\sqrt{2}(2L_G + L_H) \left(
\frac{\|\Bb\| }{ \lambda_{\min}(\Bb) }
\right)^{1/2}
\cdot \gamma_1
\le
(\lambda_1 - \lambda_2) \cdot\frac{  \lambda_{\min} (\Bb) }{ \|\Bb\| }
,
$$
where we applied $\| P_{\cT(\xb)} \vv_1 \| \le 1$.
This completes the proof of Lemma combining with the definition of $\betai$ in \eqref{GEVparam}.
\end{enumerate}

\end{proof}

\pb\section{Deferred Auxiliary Proofs of \S\ref{sec:localconvexity}}\label{sec_proof-localconvexity}
We collect the deferred auxiliary proofs of \S\ref{sec:localconvexity}.

\pb\subsection{Proof of Lemma \ref{lemm:Gtruncate}}\label{sec_proof,lemm:Gtruncate}

\begin{proof}[Proof of Lemma \ref{lemm:Gtruncate}]
Since $M = \cV \log^{\frac{1}{\alpha}}\eps^{-1}$, we have from Assumption \ref{ass:se} that for each $t \ge 1$,
\begin{align*}
\PP\left(
\left\|\Gamma(\vv_{t - 1}; \bzeta_t)\right\|
>
M
\right)
&=
\PP\left(
\exp\left(
\frac{\left\|\Gamma(\vv_{t - 1}; \bzeta_t)\right\|^\alpha}{\cV^\alpha}
\right)
>
\exp\left(
\frac{M^\alpha}{\cV^\alpha}
\right)
\right)
\\&\le
\exp\left(
- \frac{M^\alpha}{\cV^\alpha}
\right) \Exs\exp\left(
\frac{\left\|\Gamma(\vv_{t - 1}; \bzeta_t)\right\|^\alpha}{\cV^\alpha}
\right)
\le
2\eps
.
\end{align*}
where we apply the Markov inequality and Assumption \ref{ass:se} (with law of total expectation applied).
Taking a union bound,
$$
\PP(\cT_M \le T_\eta^*)
\le
\sum_{t = 1}^{T_\eta^*} \PP\left(
\left\|\Gamma(\vv_{t - 1}; \bzeta_t)\right\|
>
M
\right)
\le
2 T_\eta^* \eps
.
$$
\end{proof}

\pb\subsection{Proof of Lemma \ref{lemm:smooth}}\label{sec_proof,lemm:smooth}

\begin{proof}[Proof of Lemma \ref{lemm:smooth}]
For all $\uu, \vv \in \cS^{d - 1}$, we have
\begin{align*}
\|g(\uu) - g(\vv)\|
&\le
\|\Ib - \uu \uu^\top\| \|\nabla F(\uu) - \nabla F(\vv)\| + \|\vv \vv^\top - \uu \uu^\top\| \|\nabla F(\vv)\|
\\&\le
1 \cdot L_K \|\uu - \vv\| + 2 \|\uu - \vv\| \cdot L_F
\\&=
(L_K + 2 L_F) \|\uu - \vv\|
,\\
\|\cH(\uu) - \cH(\vv)\|
&\le
\|\nabla^2 F(\uu) - \nabla^2 F(\vv)\| + (\|\uu - \vv\| \|\nabla F(\uu)\| + \|\vv\| \|\nabla F(\uu) - \nabla F(\vv)\|) \|\Ib\|
\\&\le
L_Q \|\uu - \vv\| + (\|\uu - \vv\| \cdot L_F + 1 \cdot L_K \|\uu - \vv\|) \cdot 1
\\&=
(L_Q + L_F + L_K) \|\uu - \vv\|
,\\
\|\cN(\uu) - \cN(\vv)\|
&\le
\|\uu - \vv\| (\|\nabla F(\uu)\| + \|\nabla^2 F(\uu)\| \|\uu\|)
\\&\qquad+
\|\vv\| (\|\nabla F(\uu) - \nabla F(\vv)\| + \|\nabla^2 F(\uu) - \nabla^2 F(\vv)\| \|\uu\| + \|\nabla^2 F(\vv)\| \|\uu - \vv\|)
\\&\le
\|\uu - \vv\| (L_F + L_K \cdot 1) + 1 \cdot (L_K \|\uu - \vv\| + L_Q \|\uu - \vv\| \cdot 1 + L_K \cdot \|\uu - \vv\|)
\\&=
(L_F + 3L_K + L_Q) \|\uu - \vv\|
,\\
\|\cH(\vv)\|
&\le
\|\nabla^2 F(\vv)\| + \|\vv\| \|\nabla F(\vv)\| \|\Ib\|
\le
L_K + 1 \cdot L_F \cdot 1
=
L_K + L_F
.
\end{align*}
which implies that $g(\vv)$ is $(L_G \equiv L_K + 2 L_F)$-Lipschitz, $\cH(\vv)$ is $(L_H \equiv L_Q + L_F + L_K)$-Lipschitz, $\cN(\vv)$ is $(L_N \equiv L_F + 3L_K + L_Q)$-Lipschitz and $\|\cH(\vv)\| \le B_H \equiv L_F + L_K$ within $\{\vv : \|\vv\| \le 1, \|\vv - \vv^*\| \le \lipr\}$.

\end{proof}

\pb\subsection{Proof of Lemma \ref{lemm:witexpress}}\label{sec_proof,lemm:witexpress}

\begin{proof}[Proof of Lemma \ref{lemm:witexpress}]
We have by a Taylor series expansion that for any $y \in \real$ satisfying $|y| \le 1/2$
$$
\left| (1 - y)^{-1 / 2} - 1 - \frac{y}{2} \right|
 \le 
\frac{3y^2}{8} \sum_{k= 0}^\infty |y|^k
\le
\frac{3y^2}{4}
.
$$
When $\eta \le 1/(5M)$, on the event $(\|\Gamma(\vv_{t - 1}; \bzeta_t)\| \le M)$, by letting $y = 2\eta \vv_{t - 1}^\top \Gamma(\vv_{t - 1}; \bzeta_t) - \eta^2 \|\Gamma(\vv_{t - 1}; \bzeta_t)\|^2$ we have
$$
|y| 
 \le 
2\eta \left| \vv_{t - 1}^\top \Gamma(\vv_{t - 1}; \bzeta_t) \right|
 + 
\eta^2 \|\Gamma(\vv_{t - 1}; \bzeta_t)\|^2
 \le
2 \eta M + \eta^2 M^2 \le (11/5)\eta M < 1/2
,
$$
and hence combining the above two displays gives
\beq\label{eq:taylor_app}
\begin{aligned}
 &\quad
\left|
\|\vv_{t - 1} - \eta \Gamma(\vv_{t - 1}; \bzeta_t)\|^{-1} - 1 - \eta \vv_{t - 1}^\top \Gamma(\vv_{t - 1}; \bzeta_t)
\right|
 \\&\le
\left|
\left( 1 - 2\eta \vv_{t - 1}^\top \Gamma(\vv_{t - 1}; \bzeta_t) + \eta^2 \|\Gamma(\vv_{t - 1}; \bzeta_t)\|^2\right)^{-1/2}
 - 1 - \eta \vv_{t - 1}^\top \Gamma(\vv_{t - 1}; \bzeta_t)
\right|
 \\&\le
\left| \left( 1 - y\right)^{-1/2} - 1 - \frac{ y}{2} \right|
+
\frac{ \eta^2 \|\Gamma(\vv_{t - 1}; \bzeta_t)\|^2}{2}
 \\&\le
\frac{3y^2}{4} + \frac{1}{2}\eta^2 M^2
 \le
\frac{3}{4} \cdot\frac{121}{25}\eta^2 M^2+ \frac{1}{2} \eta^2 M^2
 \le
5 \eta^2 M^2
.
\end{aligned}\eeq

By defining 
\beq\label{eq:bxi}
\bxi_t
=
(\Ib - \vv_{t - 1} \vv_{t - 1}^\top) (\Gamma(\vv_{t - 1}; \bzeta_t) - D(\vv_{t - 1}) \nabla F(\vv_{t - 1}))
,
\eeq
and
\beq\label{eq:bQ}
\begin{split}
\bQ_t
&=
\eta^{-2} \cdot \left(
\|\vv_{t - 1} - \eta \Gamma(\vv_{t - 1}; \bzeta_t)\|^{-1} - 1 - \eta \vv_{t - 1}^\top \Gamma(\vv_{t - 1}; \bzeta_t)
\right) (\vv_{t - 1} - \eta \Gamma(\vv_{t - 1}; \bzeta_t))
\\&\quad-
(\vv_{t - 1}^\top \Gamma(\vv_{t - 1}; \bzeta_t)) \Gamma(\vv_{t - 1}; \bzeta_t)
,
\end{split}
\eeq
the update formula \eqref{PSSGD} is equivalent to
\beq\label{eq:bxi_bQ_decompose}
\vv_t
=
\vv_{t - 1} - \eta D(\vv_{t - 1}) g(\vv_{t - 1}) + \eta \bxi_t + \eta^2 \bQ_t
.
\eeq
Using \eqref{eq:taylor_app}, we have
$$
\|\bQ_t\|
\le
\eta^{-2} \cdot 5 \eta^2 M^2 \cdot (1 + \eta M) + M^2
\le
7 M^2
.
$$

Recall that we denote $D = D(\vv^*), \cH_* = \cH(\vv^*), \cN_* = \cN(\vv^*)$.
By defining
\beq\label{eq:bR}
\bR_t
=
D (\cH_* + \cN_*) (\vv_{t - 1} - \vv^*) - D(\vv_{t - 1}) g(\vv_{t - 1})
,
\eeq
we have
$$
\vv_t
=
\vv_{t - 1} - \eta D (\cH_* + \cN_*) (\vv_{t - 1} - \vv^*) + \eta \bxi_t + \eta \bR_t + \eta^2 \bQ_t
.
$$
Since $(\Ib - \vv_{t - 1} \vv_{t - 1}^\top)$ is $\cF_{t - 1}$-measurable, we know that $\Exs[\bxi_t \mid \cF_{t - 1}] = 0$ and hence $\{\bxi_t\}$ is a vector-valued martingale difference sequence.
Additionally, we have $\|\Ib - \vv_{t - 1} \vv_{t - 1}^\top\| \le 1$, and hence from Assumption \ref{ass:se} and Lemma \ref{lemm:psialpha} we know
$$
\Exs\exp\left(
\frac{\|\bxi_t\|^\alpha}{(G_\alpha \cV)^\alpha}
\right)
\le
\Exs\exp\left(
\frac{\|\Gamma(\vv_{t - 1}; \bzeta_t) - D(\vv_{t - 1}) \nabla F(\vv_{t - 1})\|^\alpha}{(G_\alpha \cV)^\alpha}
\right)
\le
2
$$
which implies that $\bxi$ is $\alpha$-sub-Weibull with parameter $G_\alpha \cV$.

\pb
Finally, we apply the mean-value theorem using \eqref{eq:nabla_g} and $g(\vv^*) = 0$ to obtain
\begin{align*}
&\,
\|\bR_t\|
=
\|D (\cH_* + \cN_*) (\vv_{t - 1} - \vv^*) - D(\vv_{t - 1}) g(\vv_{t - 1})\|
\\&\le
D \left\|
(\cH_* + \cN_*) (\vv_{t - 1} - \vv^*) - \int_0^1 \cH(\vv^* + \theta (\vv_{t - 1} - \vv^*)) + \cN(\vv^* + \theta (\vv_{t - 1} - \vv^*)) \ud \theta~ (\vv_{t - 1} - \vv^*)
\right\|
\\&\quad+
\|D - D(\vv_{t - 1})\| \|g(\vv_{t - 1})\|
\\&\le
D (L_H + L_N) \|\vv_{t - 1} - \vv^*\|^2 + L_D L_G \|\vv_{t - 1} - \vv^*\|^2
\end{align*}
where we use the Lipschitz continuity of $D(\vv), g(\vv), \cH(\vv), \cN(\vv)$.
This completes the proof of Lemma \ref{lemm:witexpress}.
\end{proof}

\pb\subsection{Proof of Lemma \ref{lemm:properties}}\label{sec_proof,lemm:properties}

\pb
\begin{proof}[Proof of Lemma \ref{lemm:properties}]
Under initialization condition \eqref{eq:warm}, we have the following:
\begin{enumerate}[label=(\roman*)]
\item
For all unit vector $\vv$, since $\|\vv\| = \|\vv^*\| = 1$ we have
$$
\|(\vv^* {\vv^*}^\top) (\vv - \vv^*)\|
=
- {\vv^*}^\top (\vv - \vv^*)
=
\frac12 \|\vv\|^2 - {\vv^*}^\top \vv + \frac12 \|\vv^*\|^2
=
\frac12 \|\vv - \vv^*\|^2
.
$$
Because
$$
\left(
(\vv^* {\vv^*}^\top) (\vv - \vv^*)
\right)^\top \left(
(\Ib - \vv^* {\vv^*}^\top) (\vv - \vv^*)
\right)
=
0
,
$$
by the Pythagorean theorem we have
$$
\|(\vv^* {\vv^*}^\top) (\vv - \vv^*)\|^2 + \|(\Ib - \vv^* {\vv^*}^\top) (\vv - \vv^*)\|^2
=
\|\vv - \vv^*\|^2
$$
Combining the above equalities and plugging in $\vv = \vv_t$ gives
$$
\|\Delta_t\|^2
=
\|\vv_t - \vv^*\|^2 - \frac14 \|\vv_t - \vv^*\|^4
,
$$
which admits the following solution given $\vv_t^\top \vv^* \ge 0$:
$$
\|\vv_t - \vv^*\|^2
=
2 - \sqrt{4 - 4 \|\Delta_t\|^2}
,
$$
and hence
$$
\|\Delta_t\|^2
\le
\|\vv_t - \vv^*\|^2
=
\frac{4 \|\Delta_t\|^2}{2 + \sqrt{4 - 4 \|\Delta_t\|^2}}
\le
2 \|\Delta_t\|^2
.
$$

\pb\item
Under initialization condition \eqref{eq:warm}, for all $\uu \in \cT(\vv^*)$, we have $\uu^\top \cH_* \uu \ge \alphai \|\uu\|^2$.
Hence for $\eta \le 1 / (D B_H)$, we have
\beq\label{eq:M_property_mid}
\|(\Ib - \eta D \cM_*)^{1 / 2} \uu\|
\le
(1 - \eta D \alphai)^{1 / 2} \|\uu\|
.
\eeq
By noticing that $(\Ib - \eta D \cM_*)^{(t - 1) / 2} \uu \in \cT(\vv^*)$, for all $t \ge 1$, we could inductively plug in $(\Ib - \eta D \cM_*)^{(t - 1) / 2} \uu$ to $\uu$ in \eqref{eq:M_property_mid} and obtain for each $t \ge 0$
$$
\|(\Ib - \eta D \cM_*)^t \uu\|
\le
(1 - \eta D \alphai)^t \|\uu\|
.
$$
\end{enumerate}
\end{proof}

\pb\subsection{Proof of Lemma \ref{lemm:Delta_express}}\label{sec_proof,lemm:Delta_express}

\begin{proof}[Proof of Lemma \ref{lemm:Delta_express}]
By left multiplying \eqref{wit} in Lemma \ref{lemm:witexpress} by $(\Ib - \vv^* {\vv^*}^\top)$ and noticing $(\Ib - \vv^* {\vv^*}^\top) \cN_* = 0$, we obtain
\begin{align*}
\Delta_t
&=
\Delta_{t - 1}
-
\eta D (\Ib - \vv^* {\vv^*}^\top) \cH_* (\vv_{t - 1} - \vv^*)
+
\eta (\Ib - \vv^* {\vv^*}^\top) \bxi_t\\
&\quad
+
\eta (\Ib - \vv^* {\vv^*}^\top) \bR_t
+
\eta^2 (\Ib - \vv^* {\vv^*}^\top) \bQ_t
.
\end{align*}
We have the decomposition
$$
(\Ib - \vv^* {\vv^*}^\top) \cH_* (\vv_{t - 1} - \vv^*)
=
(\Ib - \vv^* {\vv^*}^\top) \cH_* \Delta_t
+
(\Ib - \vv^* {\vv^*}^\top) \cH_* \cdot (\vv^* {\vv^*}^\top) (\vv_{t - 1} - \vv^*)
,
$$
where $(\Ib - \vv^* {\vv^*}^\top) \cH_* \Delta_t = \cM_* \Delta_t$, and based on Lemma \ref{lemm:properties} and $\|\cH_*\| \le B_H$,
$$
\|(\Ib - \vv^* {\vv^*}^\top) \cH_* \cdot (\vv^* {\vv^*}^\top) (\vv_{t - 1} - \vv^*)\|
\le
\frac{B_H}{2} \|\vv_{t - 1} - \vv^*\|^2
.
$$
We set
$$\begin{aligned}
\bchi_t
&=
(\Ib - \vv^* {\vv^*}^\top) \bxi_t
,\\
\bS_t
&=
(\Ib - \vv^* {\vv^*}^\top) \bR_t
-
D \cdot (\Ib - \vv^* {\vv^*}^\top) \cH_* \cdot (\vv^* {\vv^*}^\top) (\vv_{t - 1} - \vv^*)
,\\
\bP_t
&=
(\Ib - \vv^* {\vv^*}^\top) \bQ_t
.
\end{aligned}$$
Then by combining all of the results above, we have
$$
\Delta_t
=
\left(\Ib - \eta D \cM_*\right) \Delta_{t - 1}
+
\eta \bchi_t
+
\eta \bS_t
+
\eta^2 \bP_t
,
$$
which proves \eqref{eq:Delta_update}.
The rest of Lemma \ref{lemm:Delta_express} can be easily verified in steps similar to the proof of Lemma \ref{lemm:witexpress}.

\end{proof}

\pb\subsection{Proof of Lemma \ref{lemm:coupling1}}\label{sec_proof,lemm:coupling1}

\begin{proof}[Proof of Lemma \ref{lemm:coupling1}]
For $t = 0$ the lemma holds by definition.
In general if it holds for $t - 1$ then from the definitions in \eqref{barRnk} we have on $(t < \cT_M)$ that $\tbS_s = \bS_s, \tbP_s = \bP_s$ for all $s \le t$, so the conclusion holds for $t$.
Iteratively applying \eqref{geomUhat} we obtain \eqref{grad}, which concludes our lemma.
\end{proof}

\pb\subsection{Proof of Lemma \ref{lemm:coord}}\label{sec_proof,lemm:coord}

\begin{proof}[Proof of Lemma \ref{lemm:coord}]
For any fixed $t \ge 0$, we estimate each term of \eqref{grad} which we repeat here
\beq\tag{\eqref{grad}}
\begin{aligned}
\tDelta_t
&=
\left(\Ib - \eta D \cM_*\right)^t \Delta_0
 +
\eta \sum_{s = 1}^t \left(\Ib - \eta D \cM_*\right)^{t - s} \bchi_s 
\notag\\&\quad
+
\eta \sum_{s = 1}^t \left(\Ib - \eta D \cM_*\right)^{t - s} \tbS_s
+
\eta^2 \sum_{s = 1}^t \left(\Ib - \eta D \cM_*\right)^{t - s} \tbP_s
.
\end{aligned}\eeq
For the first term on the right hand of \eqref{grad}, since $\bchi_s \in \cT(\vv^*)$, \eqref{eq:M_property} in Lemma \ref{lemm:properties} implies $\|(\Ib - \eta D \cM_*)^{t - s} \bchi_s\| \le (1 - \eta D \alphai)^{t - s} \|\bchi_s\|$.
Hence we have $\|(\Ib - \eta D \cM_*)^{t - s} \bchi_s\|_{\psi_\alpha} \le (1 - \eta D \alphai)^{t - s} \|\bchi_s\|_{\psi_\alpha} \le (1 - \eta D \alphai)^{t - s} G_\alpha \cV$ and
$$
\sum_{s = 1}^t \left\|
\eta (\Ib - \eta D \cM_*)^{t - s} \bchi_s
\right\|_{\psi_\alpha}^2
\le
\eta^2 \sum_{s = 1}^t (1 - \eta D \alphai)^{2(t - s)} G_\alpha^2 \cV^2
\le
\frac{G_\alpha^2 \cV^2}{D \alphai} \cdot \eta
$$
Modifying the results in \citet{fan2012large} provides a concentration inequality for $\alpha$-sub-Weibull random vectors, which gives%
\footnote{A similar concentration inequality method for the scalar case is adopted by \citet{li2021stochastic}.}
$$\begin{aligned}
\PP\left(
\left\| \eta \sum_{s = 1}^t (\Ib - \eta D \cM_*)^{t - s} \bchi_s \right\|
	\ge
\frac{8G_\alpha \cV}{\sqrt{D \alphai}} \log^{\frac{\alpha + 2}{2\alpha}} \eps^{-1} \cdot \eta^{1 / 2}
\right)
	\le
\left(
12 + 8 \left(\frac{3}{\alpha}\right)^{\frac{2}{\alpha}} \log^{- \frac{\alpha + 2}{\alpha}} \eps^{-1}
\right) \eps
.
\end{aligned}$$
For the second term on the right-hand side of \eqref{grad}, by applying \eqref{eq:M_property} in Lemma \ref{lemm:properties} and using Lemma \ref{lemm:Delta_express}, given $\|\vv_{s - 1} - \vv^*\| \le \radius$ for all $s = 1,\dots, t$ we have,
\beq\label{term1T}
\begin{aligned}
\left\| \eta \sum_{s = 1}^t \left( \Ib - \eta D \cM_* \right)^{t - s} \tbS_s \right\|
	\le
\eta \sum_{s = 1}^t (1 - \eta D \alphai)^{t - s} \cdot \rhoi \radius^2
	\le
\frac{\rhoi \radius^2}{D \alphai}
.
\end{aligned}\eeq
For the third term on the right-hand side of \eqref{grad}, from Lemma \ref{lemm:Delta_express} we know $\|\tbP_t\| \le 7M^2$ and
$$
\left\| \eta^2 \sum_{s = 1}^t \left( \Ib - \eta D \cM_* \right)^{t - s} \tbP_s \right\|
	\le
\eta^2 \sum_{s = 1}^t (1 - \eta D \alphai)^{t - s} \cdot 7M^2
	=
\frac{7\cV^2}{D \alphai} \log^{\frac{2}{\alpha}}\eps^{-1} \cdot \eta
,
$$
where we use the definition of $M$ in \eqref{eq:M}.
The lemma is concluded by combining the above three items and taking union bound on probability.
\end{proof}

\pb\subsection{Proof of Lemma \ref{lemm:radius}}\label{sec_proof,lemm:radius}

\begin{proof}[Proof of Lemma \ref{lemm:radius}]
From the given assumptions, under scaling condition \eqref{eq:eta_scaling}, we have
$$
\radius
=
2 \max\left\{
\|\Delta_0\|
,~
\frac{2^7 G_\alpha \cV}{\sqrt{D \alphai}} \log^{\frac{\alpha + 2}{2\alpha}} \eps^{-1} \cdot \eta^{1 / 2}
\right\}
\le
\frac{D \alphai}{16 \rhoi}.
$$
We let event $\kJ$ be \eqref{xt_concentration} holding for each $t \in [0, T]$, i.e.
$$
\left\|
\tDelta_t - (\Ib - \eta D \cM_*)^t \Delta_0
\right\|
\le
\frac{8G_\alpha \cV}{\sqrt{D \alphai}} \log^{\frac{\alpha + 2}{2\alpha}} \eps^{-1} \cdot \eta^{1 / 2}
+
\frac{\rhoi \radius^2}{D \alphai}
+
\frac{7 \cV^2}{D \alphai} \log^{\frac{2}{\alpha}}\eps^{-1} \cdot \eta
.
$$
Then on event $\kJ$, under scaling condition \eqref{eq:eta_scaling}, because $\|\Delta_0\| \le \frac{\radius}{2}$, for each $t \in [0, T]$ we have
$$
\|\tDelta_t\|
\le
\|\Delta_0\|
+
\frac{16 G_\alpha \cV}{\sqrt{D \alphai}} \log^{\frac{\alpha + 2}{2\alpha}} \eps^{-1} \cdot \eta^{1 / 2}
+
\frac{\rhoi \radius^2}{D \alphai}
\le
\frac{\radius}{2}
+
\frac{\radius}{16}
+
\frac{\radius}{16}
\le
\radius
.
$$
Applying Lemma \ref{lemm:coord} and taking a union bound gives
$$
\PP(\kJ)
\ge
1 - \left(
12 + 8 \left(\frac{3}{\alpha}\right)^{\frac{2}{\alpha}} \log^{- \frac{\alpha + 2}{\alpha}} \eps^{-1}
\right) T \eps.
$$
Furthermore, using \eqref{eq:M_property} in Lemma \ref{lemm:properties} and definition of $T_\eta^*$ in \eqref{Tstareta}, if $T_\eta^* \in [0, T]$, on event $\kJ$ we have at time $T_\eta^*$
$$
\|\tDelta_{T_\eta^*}\|
	\le
\|(\Ib - \eta D \cM_*)^{T_\eta^*} \Delta_0\|
+
\frac{16 G_\alpha \cV}{\sqrt{D \alphai}} \log^{\frac{\alpha + 2}{2\alpha}} \eps^{-1} \cdot \eta^{1 / 2}
+
\frac{\rhoi \radius^2}{D \alphai}
	\le
\frac{\radius}{8} + \frac{\radius}{16} + \frac{\radius}{16}
	\le
\frac{\radius}{4}
.
$$
In Lemma \ref{lemm:coupling1} we have shown that, on the event $(T < \cT_M)$, we have $\tDelta_t = \Delta_t$.
In Lemma \ref{lemm:Gtruncate}, we have proved $\PP(T < \cT_M) \ge 1 - 2 T \eps$.
Together with Lemma \ref{lemm:coord}, we take an intersection and obtain
$$
\PP(\kJ \cap (T < \cT_M))
\ge
1 - \PP(\kJ^c) - \PP(T \ge \cT_M)
\ge
1 - \left(
14 + 8 \left(\frac{3}{\alpha}\right)^{\frac{2}{\alpha}} \log^{- \frac{\alpha + 2}{\alpha}} \eps^{-1}
\right) T \eps
.
$$
At this point we have proved all elements in Lemma \ref{lemm:radius}.
\end{proof}

\end{document}